\documentclass{article}

\usepackage{silence}
\usepackage[preprint]{include/neurips_2023}

\usepackage[utf8]{inputenc} 
\usepackage[T1]{fontenc}    
\usepackage{hyperref}       
\usepackage{url}            
\usepackage{booktabs}       
\usepackage{amsfonts}       
\usepackage{nicefrac}       
\usepackage{microtype}      
\usepackage{xcolor}         

\usepackage{amsmath}
\usepackage{graphicx}
\usepackage{enumitem}
\usepackage{algorithm}
\usepackage{algorithmic}
\usepackage{multirow}

\usepackage{amsthm}
\theoremstyle{plain}
\newtheorem{theorem}{Theorem}[section]
\newtheorem{definition}[theorem]{Definition}
\newtheorem{example}[theorem]{Example}

\DeclareMathOperator*{\argmax}{arg\,max}
\DeclareMathOperator*{\argmin}{arg\,min}

\title{Robust Universal Adversarial Perturbations}

\author{%
  Changming Xu\\
  Department of Computer Science\\
  University of Illinois Urbana-Champaign\\
  Champaign, IL 61801\\
  \texttt{cxu23@illinois.edu}\\
  \And
  Gagandeep Singh\\
  Department of Computer Science\\
  University of Illinois Urbana-Champaign\\
  Champaign, IL 61801\\
  \texttt{ggnds@illinois.edu}\\
}

\begin{document}

\maketitle

\begin{abstract}
  Universal Adversarial Perturbations (UAPs) are imperceptible, image-agnostic vectors that cause deep neural networks (DNNs) to misclassify inputs with high probability. In practical attack scenarios, adversarial perturbations may undergo transformations such as changes in pixel intensity, scaling, etc. before being added to DNN inputs. Existing methods do not create UAPs robust to these real-world transformations, thereby limiting their applicability in practical attack scenarios. In this work, we introduce and formulate UAPs robust against real-world transformations. We build an iterative algorithm using probabilistic robustness bounds and construct such UAPs robust to transformations generated by composing arbitrary sub-differentiable transformation functions. We perform an extensive evaluation on the popular CIFAR-10 and ILSVRC 2012 datasets measuring our UAPs' robustness under a wide range common, real-world transformations such as rotation, contrast changes, etc. We further show that by using a set of primitive transformations our method can generalize well to unseen transformations such as fog, JPEG compression, etc. Our results show that our method can generate UAPs up to $23\%$ more robust than state-of-the-art baselines. 
\end{abstract}

 \section{Introduction}

Deep neural networks (DNNs) have achieved impressive results in many application domains such as natural language processing~\citep{cnnspeech, gpt3}, medicine~\citep{skincancer, dlmedguide}, and computer vision~\citep{vgg, inception}. Despite their performance, they can be fragile in the face of adversarial 
perturbations: small imperceptible changes added to a correctly classified input that make a DNN misclassify. 
%
While there is a large amount of work on generating adversarial perturbations~\citep{intriguing, fgsm, deepfool, pgd, cw, ijcai:18, Dong:18, croce:19, wang:19, distributionally:19, andriushchenko:19, tramer:20}, these works depend upon unrealistic assumptions about the power of the attacker:  
the attacker knows the DNN input in advance, generates input-specific perturbations in real-time and \textit{exactly} combines the perturbation with the input before being processed by the DNN. Thus, we argue that these threat models are not realizable in many real-world applications.

\textbf{Practically feasible adversarial perturbations.} In this work, we consider a more practical adversary to reveal real-world vulnerabilities of state-of-the-art DNNs. We assume that the attacker (i) does not know the DNN inputs in advance, (ii) can only transmit additive adversarial perturbations, and (iii) their transmitted perturbations are susceptible to modification due to real-world effects. 
Examples of attacks in our threat model include adding stickers to the cameras for fooling image classifiers~\citep{sticker:19} or transmitting perturbations over the air for deceiving audio classifiers~\citep{music:19}. Note that our threat model is distinct from directly generating adversarial examples (i.e. creating physical adversarial objects ~\citep{turtle}) which require access to the original input. 

\begin{figure*}[ht]
\begin{center}
\centerline{\includegraphics[width=0.8\linewidth]{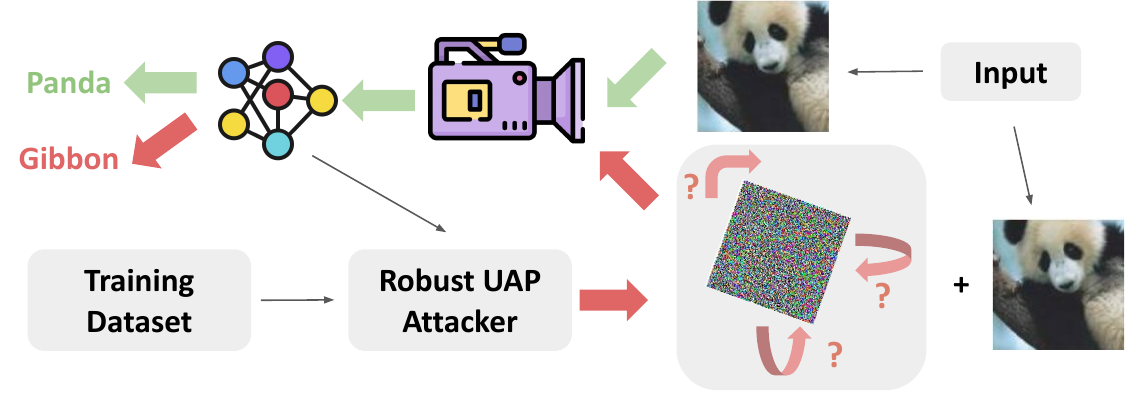}}
\caption{Robust UAP Threat Model: Input Agnostic + Robust to Transmission Transformation}
\label{fig:ruap_threat}
\end{center}
\vskip -0.3in
\end{figure*}

The first two requirements in our threat model can be fulfilled by generating
Universal Adversarial Perturbations (UAPs)~\citep{uap}. Here the attacker can train a single adversarial perturbation that has a high probability of being adversarial on all inputs in the training distribution. However, as our experimental results show, the generated UAPs need to be combined with the DNN inputs precisely, otherwise they fail to remain adversarial. In practice, changes to UAPs are likely due to real-world effects. 
For example, the stickers applied to a camera can undergo changes in contrast due to weather conditions or the transmitted perturbation in audio can change due to noise in the transmission channel. 
This non-robustness reduces the efficiency of practical attacks created with existing methods~\citep{uap,uat,sticker:19,music:19}.

\textbf{This work: Robust UAPs.}
To overcome the above limitation, we propose the concept of robust UAPs: perturbations that have a high probability of remaining adversarial on inputs in the training distribution even after applying a set of real-world transformations. 
The optimization problem in generating robust UAPs~\citep{uap} is made challenging as we are looking for perturbations that are adversarial for a set of inputs as well as to a set of potentially unknown transformations applied to the perturbations. To address this challenge, we make the following main \textbf{contributions}:  


\begin{itemize}[leftmargin=*]
\item We introduce \textit{Robust UAPs} and formulate their generation as an optimization problem. We separate our threat model into two scenarios depending on whether the transformation set is known apriori.
\item We design a new method, \texttt{RobustUAP}, for constructing robust UAPs. Our method is general and constructs UAPs robust to any transformations generated by composing arbitrary sub-differentiable transformation functions. We provide an algorithm for computing provable probabilistic bounds on the robustness of our UAPs against many practical transformations. We show that in the vision domain we can use a set of primitive transforms (adapted from \citet{prime}) to create \textit{Universally} Robust UAPs. 
\item We perform an extensive evaluation of our method on state-of-the-art models for the popular CIFAR-10~\citep{cifar} and ILSVRC 2012~\citep{imagenet} datasets. We compare the robustness of our  UAPs under compositions of challenging real-world transformations, such as rotation, contrast change, etc. We show that on both datasets, the UAPs generated by \texttt{RobustUAP} are significantly more robust, achieving up to $23\%$ more robustness, than the UAPs generated from the baselines. 

\end{itemize}

Our work is complementary to the development of real-world attacks~\citep{music:19,sticker:19} in various domains, which require modeling how the universal perturbations change during transmission. \texttt{RobustUAP} can improve the efficiency of such attacks by constructing perturbations that are more robust to real-world transformations than with existing algorithms~\citep{uap,uat,music:19,sticker:19}. Our results using primitive transformations in vision suggest that we can forego domain specific modeling in other domains if we can find a good set of primitives for that domain.
\section{Background}\label{sec:background}

In this section, we provide necessary background definitions and notation used in the rest of our work. For the remainder of the paper, let $\mu \subset \mathbb{R}^d$ be the input data distribution, $\mathbf{x} \in \mu$ be an input point with the corresponding true label $y \in \mathbb{R}$, and $f:\mathbb{R}^d \to \mathbb{R}^{d'}$ be our target classifier. For ease of notation, we define $f_k(\mathbf{x})$ to be the $k^{\text{th}}$ element of $f(\mathbf{x})$ and allow $\hat{f}(\mathbf{x}) = \argmax_k f_k(\mathbf{x})$ to directly refer to the classification label. We use $\mathbf{v}$ to reference image specific perturbations and $\mathbf{u}$ to reference universal adversarial perturbations, $\mathbf{v_r}$ and $\mathbf{u_r}$ refer to the robust variants and will be defined in Sec. \ref{sec:ruap}. We provide formal definitions in Appendix \ref{appendix:definitions}.

\textbf{Adversarial Examples and Perturbations. }
An \textit{adversarial example} is a misclassified data point that is \textit{close} (in some norm) to a correctly classified data point~\citep{fgsm, pgd, cw}. 
In this paper, we consider examples $\mathbf{x'}$ generated as $\mathbf{x'}=\mathbf{x}+\mathbf{v}$ where $\mathbf{v}$ is an \textit{adversarial perturbation}. 
%
%
%





\textbf{Universal Adversarial Perturbations. }UAPs are single vector, input-agnostic perturbations~\citep{uap}. They differ from traditional adversarial attacks, which create perturbations dependent on each input sample. To measure UAP performance, we introduce the notion of universal adversarial success rate ($\text{ASR}_U$), which measures the probability that a perturbation $\mathbf{u}$ when added to $\mathbf{x}$, sampled from $\mu$, causes a change in classification under $f$. Thus a perturbation, $u$, is a UAP given two conditions: its $\text{ASR}_U$ is greater than a given threshold, $\gamma$, and its norm is small. If an additive perturbation has a small $l_p$-normit does not affect the semantic content of the image as it appears as noise.
We pose the construction of UAPs as an expectation minimization problem:

\vspace{-0.05in}
\begin{equation}\label{eq:uap_opt}
    \argmin_{u}\mathbb{E}_{\mathbf{x}\sim\mu}[\delta (\hat{f}(\mathbf{x}+\mathbf{u}), \hat{f}(\mathbf{x}))] \text{ s.t. } ||\mathbf{u}||_p < \epsilon
\end{equation}
\vspace{-0.05in}

where $\delta$ is the Kronecker Delta function~\citep{calc}.



\begin{figure*}[ht]
\begin{center}
\centerline{\includegraphics[width=0.8\linewidth]{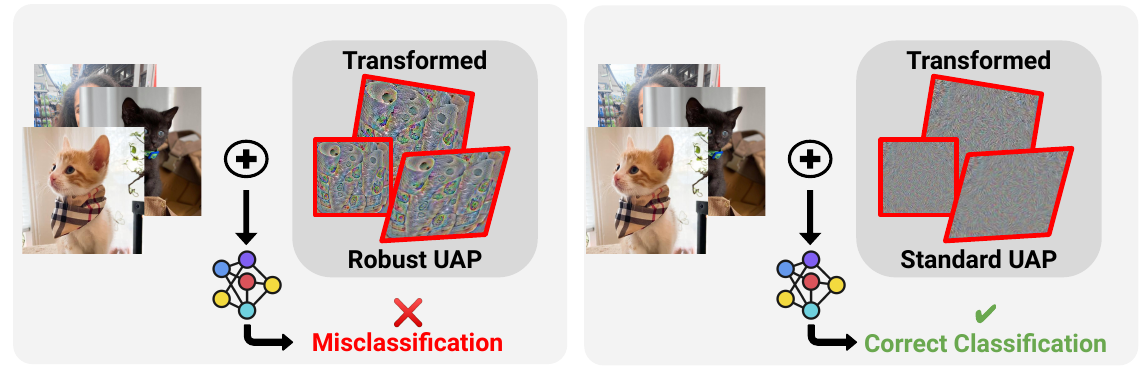}}
\caption{Robust UAPs (left) cause a classier to misclassify on \textit{most} of the data distribution even after transformations are applied on them. Standard UAPs (right) are not robust to transformations and have a low probability of remaining UAPs after transformation.}
\label{fig:semanticrobustuap}
\end{center}
\vskip -0.3in
\end{figure*}

\section{Robust Universal Adversarial Perturbations}\label{sec:ruap}


In this section, we first define our notion of transformation sets and neighborhoods in order to define robust UAPs. Here, when we are referencing transformation sets as the ones applied during transmission, if these are unknown, we detail our method for overcoming this in Section \ref{sec:unknown}. Formal definitions of all terms can be found in Appendix \ref{appendix:definitions}. 

		



\textbf{Transformation Sets and Neighborhoods. } We define a transformation set, $T$, as all transforms, $\tau$, which can be made by composing from a predefined set of bijective sub-differentiable transformation functions. The neighborhood, $N_T(v)$, of a point $v$ is all points, $v'$ reachable from $v$ using transformations from $T$.

\begin{example}
\label{ex:transform}
Let $T$ be all transformations represented by a rotation of $\pm 30^{\circ}$ and scaling of up to a factor of $2$, in this case one $\tau \in T$ could be \{rotation of $8^{\circ}$ and scaling a factor of 1.2\} in that order and $N_T(\mathbf{v})$ would include any point obtained by applying a transformation from $T$ on $\mathbf{v}$.
\end{example}


\textbf{Robust UAPs. } In order to define robust UAPs we introduce robust universal adversarial success rate. The \textit{robust universal adversarial success rate}, $\text{ASR}_R$, measures the probability that a neighbor of $\mathbf{u_r}$ is also an UAP on $\mu$, i.e. after transformation it maintains high universal ASR above some threshold $\gamma$. We note that even though $||\mathbf{u_r}||_p \leq \epsilon$, it can happen that a $\mathbf{u'_r} \in N_T(\mathbf{u_r})$ has $||\mathbf{u'_r}||_p > \epsilon$. Therefore, we require that the norm of $\mathbf{u'_r}$ is small.


\begin{definition}
\label{def:ruap}
A \textit{robust UAP}, $\mathbf{u_r}$, is one which \textit{most} points within a neighborhood of $\mathbf{u_r}$ when added to \textit{most} points in $\mu$ fool the classifier, $f$. $\mathbf{u_r}$ satisfies $||\mathbf{u_r}||_p < \epsilon$ and $\text{ASR}_R(f, \mu, T, \gamma, \mathbf{u_r}) > \zeta$.
\end{definition}

In order to construct robust UAPs, we can pose the following expectation minimization problem:

\vspace{-0.15in}
\begin{equation}\label{eq:ruap_opt}
    \argmin_{\mathbf{u_r}}\mathop{\mathbb{E}}_{\mathbf{u'_r}\in N_T(\mathbf{u_r})}[I(||\mathbf{u_r'}||<\epsilon) \times \mathop{\mathbb{E}}_{\mathbf{x}\sim\mu}[\delta(\hat{f}(\mathbf{x} + \mathbf{u_r'}), \hat{f}(\mathbf{x}))]] \text{ s.t. } ||\mathbf{u_r}||_p < \epsilon
\end{equation}
\vspace{-0.1in}

Here $I:\mathbb{R}^d \to \mathbb{R}$ denotes an indicator function. The inner expectation represents the UAP condition for the transformed perturbation $\mathbf{u'_r}$  while the outer expectation represents the neighborhood robustness condition. The composition of these conditions in Equation \ref{eq:ruap_opt} makes it computationally harder than minimizing over only the transformation set, as in EOT~\citep{turtle}, or than minimizing over only the data distribution, as in Equation \ref{eq:uap_opt}. 



\section{Generating Robust Universal Adversarial Perturbations}

In this section we discuss how we deal with both known and unknown transformation sets, then describe our approach for optimizing Equation \ref{eq:ruap_opt}. As it would be computationally prohibitive to precisely compute the expected value, we estimate the expected value per batch, $\mathbf{\hat{x}} \subset \mu$, and random set of transformations sampled from $T$, $\hat{\tau} \subset T$. We can approximate Equation \ref{eq:ruap_opt} using:

\vspace{-0.15in}
\begin{equation}
     \frac{I(||\hat{\tau}_j(\mathbf{u_r})||<\epsilon)}{|\mathbf{\hat{x}}|\times |\hat{\tau}|} \sum_{i = 1}^{|\mathbf{\hat{x}}|}\sum_{j = 1}^{|\hat{\tau}|} L[f(\mathbf{\hat{x}_i} + \hat{\tau}_j(\mathbf{u_r})), f(\mathbf{\hat{x}_i)}] - \lambda ||\mathbf{u_r}||_p
\end{equation}
\vspace{-0.1in}

We describe our two threat models and some intuitive baselines for optimizing Equation \ref{eq:ruap_opt}. We then present our new algorithm, \texttt{RobustUAP}.

\vspace{-0.15in}
\subsection{Threat Models: Known vs Unknown Transformation Sets}\label{sec:unknown}
\vspace{-0.1in}

In the above section, we have assumed that the transformations applied during transmission is known to the attacker and used to train the UAP. However, in a real-world attack scenario the attacker may not know precisely what transformations its perturbation will undergo. In such scenarios, they may want their attack to be robust to unseen perturbations. In this case, we propose generating robust UAPs using a set of primitive transformations. For the image domain, we draw from existing work in the data augmentation area. 
In their paper, \textbf{PRI}mitives of \textbf{M}aximum \textbf{E}ntropy (PRIME), \citet{prime} define three primitive transformations: spectral, spatial, and color. Using random combinations of these transformations to train, they find that they are able to generalize well to unseen transformations such as frost, JPEG compression, motion blur, etc. We can use these primitive transformations to generate robust UAPs and we show that in practice this generates robust UAPs which generalize well to a variety of unseen transforms. In other domains, we hope that this work helps to inspire finding similar primitive transformation sets.

\vspace{-0.15in}
\subsection{Baseline Algorithms}
\vspace{-0.1in}

We propose two baseline algorithms for generating robust UAPs. The first method is momemtum based Stochastic Gradient Descent (SGD). We can directly solve Equation \ref{eq:ruap_opt} using gradient descent. The second baseline is leveraging the standard UAP algorithm from~\citet{uap}, but instead of computing an adversarial perturbation at each point, we compute a robust adversarial perturbation at each point. More details about both of these baseline algorithms can be found in Appendix \ref{appendix:baseline}. Both of these algorithms can be seen as naively combining the EoT and UAP algorithms, in the next section we describe \texttt{RobustUAP} our algorithm which takes a more principled approach at robust UAP generation.

\vspace{-0.15in}
\subsection{Robust UAP Algorithm}
\vspace{-0.1in}

The baseline algorithms have two fundamental limitations: (i) they rely on random sampling over the symbolic transformation region, but the sampling strategy does not explicitly try to maximize the robustness of the generated UAP over the entire symbolic region, and (ii) they do not estimate robustness on unsampled transformations. These baselines can be seen as naive combinations of the EoT and UAP algorithms. As a result, the baselines yield suboptimal UAPs (as confirmed by our experiments below). To overcome these fundamental limitations, we create a method to compute probabilistic bounds for expected robustness on an entire symbolic region. We leverage this method for approximating expected robustness in a new algorithm to generate robust UAPs with guarantees. 
We make a simplifying assumption that $N_T(\mathbf{u_r})$ has a well defined, sampleable probability density function (PDF) as we cannot bound robustness for arbitrary transformations. Our experiments show that even though our assumptions do not hold for all the transformation sets considered in this work, they significantly improve the robustness of our generated UAPs. Our approximation of the expected robustness relies on the following theoretical result:

\begin{theorem}
\label{thrm:robust}
Given a perturbation $\mathbf{u_r}$, a neural network $f$, a finite set of inputs $\mathbf{X}$, a set of transformations $T$, and minimum universal adversarial success rate $\gamma \in \mathbb{R}$. Let $p(\gamma) = P_{\mathbf{u_r'}\sim N_T(\mathbf{u_r})}(ASR_U(f, \mathbf{X}, \mathbf{\mathbf{u_r'}}) > \gamma)$. For $i \in 1\dots n$, let $\mathbf{u_r^i} \sim N_T(\mathbf{u_r})$ be random variables with a well defined PDF and $I:\mathbb{R}^d \to \mathbb{R}$ be the indicator function, let
\begin{equation}
\hat{p}_n(\gamma) = \frac{1}{n} \sum_{i=1}^n I(ASR_U(f, \mathbf{X}, \mathbf{u_r^i}) > \gamma)
\end{equation}
\vspace{-0.15in}

For accuracy level, $\psi \in (0,1)$, and confidence, $\phi \in (0, 1)$, where $(0,1)$ is the open interval between $0$ and $1$. If $n \geq \frac{1}{2\psi^2}\ln\frac{2}{\phi}$ then

\vspace{-0.15in}
\begin{equation}\label{eq:theorem_eq}
P(|\hat{p}_n(\gamma) - p(\gamma)| < \psi) \geq 1 - \phi
\end{equation}
\vspace{-0.15in}
\end{theorem}




The proof of this theorem can be found in Appendix \ref{appendix:theorem}. Theorem \ref{thrm:robust} states that with enough samples from the neighborhood of a perturbation, $\mathbf{u_r}$, the adversarial success rate of $\mathbf{u_r}$ on the entire neighborhood is arbitrarily close to the adversarial success rate of $\mathbf{u_r}$ on sampled transformations with probability greater than $1-\phi$. 


Leveraging Theorem \ref{thrm:robust}, we create \texttt{EstRobustness} which given accuracy, $\psi$, and confidence, $\phi$, returns the $ASR_R$ on a finite set of inputs with probabilistic robustness guarantees under the assumptions of Theorem \ref{thrm:robust}. The pseudocode for \texttt{EstRobustness} is in Appendix \ref{appendix:eralg}.


\textbf{Our algorithm: \texttt{RobustUAP}.} We leverage Theorem \ref{thrm:robust} and \texttt{EstimateRobustness} to develop \texttt{RobustUAP}, the pseudocode for which is seen in Algorithm \ref{alg:rob_uap}. Similar to the SGD baseline, we approximate the expectation in Equation \ref{eq:ruap_opt} in batches. We first sample transformations from the PDF of the neighborhood. We set the number of transformations, $n$, based on Theorem \ref{thrm:robust} to satisfy the desired confidence level and accuracy. For each gradient step, we compute the mean loss over the current batch and set of sampled transforms (line 8). 
For each set of batch and sampled transformations, instead of making a single gradient update like SGD, we use Projected Gradient Descent (PGD) to iteratively compute a more robust update to the universal perturbation and end only when the estimated robustness on the batch satisfies a given threshold (line 10). At the end of each epoch, we check the robustness across the entire training set and transformation space using \texttt{EstRobustness} (\texttt{ER}) and stop when we have reached the desired performance (line 14).

\vspace{-0.1in}
\begin{algorithm}[ht]
   \caption{Robust UAP Algorithm}
   \label{alg:rob_uap}
\begin{algorithmic}[1]
   \STATE Initialize $\mathbf{u_r} \gets 0, n \gets \lceil \frac{1}{2\psi^2}\ln\frac{2}{\phi}\rceil$
   \REPEAT
   \FOR{$\mathbf{B} \subset \mathbf{X}$}
   \STATE For $i = 1\dots n$ sample $\tau_i \sim T$
   \IF{\texttt{ER}$(f, \mathbf{B}, T, \gamma, \mathbf{u_r}, \psi, \phi) < \zeta$}
   \STATE $\Delta \mathbf{u_r} \gets 0$
   \REPEAT
   \STATE Compute $L_{\mathbf{B},\tau} = \frac{1}{|\mathbf{B}|\times n} \sum_{i = 1}^{|\mathbf{B}|}\sum_{j = 1}^{n}$ $L[f(\mathbf{B_i} + \tau_j(\mathbf{u_r} + \Delta \mathbf{u_r})), f(\mathbf{B_i})]$
   \STATE $\Delta \mathbf{u_r} = \mathcal{P}_{p, \epsilon}( \Delta \mathbf{u_r} + \alpha \text{sign}(\nabla L_{\mathbf{B},\tau}))$
   \UNTIL{\texttt{ER}$(f, \mathbf{B}, T, \gamma, \mathbf{u_r} + \Delta \mathbf{u_r}, \psi, \phi) < \zeta$}
   \STATE $\mathbf{u_r} \gets \mathcal{P}_{p, \epsilon}(\mathbf{u_r} + \Delta \mathbf{u_r})$
   \ENDIF
   \ENDFOR
   \UNTIL{\texttt{ER}$(f, \mathbf{X}, T, \gamma, \mathbf{u_r}, \psi, \phi) < \zeta$}
\end{algorithmic}
\end{algorithm}
\vspace{-0.1in}


\section{Evaluation}\label{sec:evaluation}

Our \texttt{RobustUAP} framework is applicable to all transformation sets in a variety of domains. We empirically evaluate our method \texttt{RobustUAP} and three baseline approaches (\texttt{SGD}, \texttt{StandardUAP\_RP}, \texttt{StandardUAP}~\citep{uap}) on popular models from the vision domain. We show that \texttt{RobustUAP} is more robust on both uniform random noise and compositions of real-world transformations such as rotation, scaling, etc. We did not have the hardware to print high resolution transparent stickers so we could not produce real-world results in the vision domain. We show that training \texttt{RobustUAP} on a set of primitive transforms results in a universally robust UAP which generalizes well to unseen transformations allowing for successful attacks without the need for domain specific modeling. 

\textbf{Experimental evaluation.} We consider two popular image recognition datasets: CIFAR-10\citep{cifar} and ILSVRC 2012\citep{imagenet}. We evaluate on a pretrained VGG16~\citep{vgg} and Inception-v3~\citep{inception} network on CIFAR-10 and ILSVRC 2012 respectively. For both we evaluate on a random subset (1000 images) for the test set. All experiments were performed on a desktop PC with a GeForce RTX(TM) 3090 GPU and a 16-core Intel(R) Core(TM) i9-9900KS CPU @ 4.00GHz. 

We report the results for $l_2$-norm with $\epsilon = 100$ for ILSVRC 2012 and $\epsilon = 10$ for CIFAR-10. These values were chosen based on the values presented by the original UAP paper~\citep{uap}. We use an image normalization function given by our pretrained models and thus scaled our $\epsilon$ values accordingly. We note that the $\epsilon$-values are significantly smaller than the image norms resulting in imperceptible perturbations that do not affect the semantic content of the image. Due to the hardness of the optimization problem, for the same norm value, the effectiveness of a UAP is less than input-specific perturbations; however, crafting input-specific perturbations requires making unrealistic assumptions about the power of the attacker as mentioned in the introduction and therefore we do not consider them part of our threat model which aims to generate practically feasible perturbations. We use $\psi = 0.05$ and $\phi = 0.05$ resulting in $n = 738$ for generating samples for our \texttt{RobustUAP} algorithm as well as reporting robust ASR in our evaluation. The UAPs are trained on 2,000 images, other parameters for evaluation are given in Appendix \ref{appendix:parameters}. Error bars/variances are reported in Appendix \ref{appendix:errorbars}.

\vspace{-0.15in}
\subsection{Robustness to Random Noise}\label{sec:noise}
\vspace{-0.1in}

We first generate UAPs robust against uniform random noise. Here since our neighborhood has a well-defined PDF we get the robustness guarantees from \texttt{EstimateRobustness}, in the following sections when we consider semantic and unknown transformations we won't have the same gaurantees. The results and future discussion can be found in Appendix \ref{appendix:randomnoise}.

\begin{figure*}[t]
\vskip -0.2in
     \includegraphics[width=\textwidth]{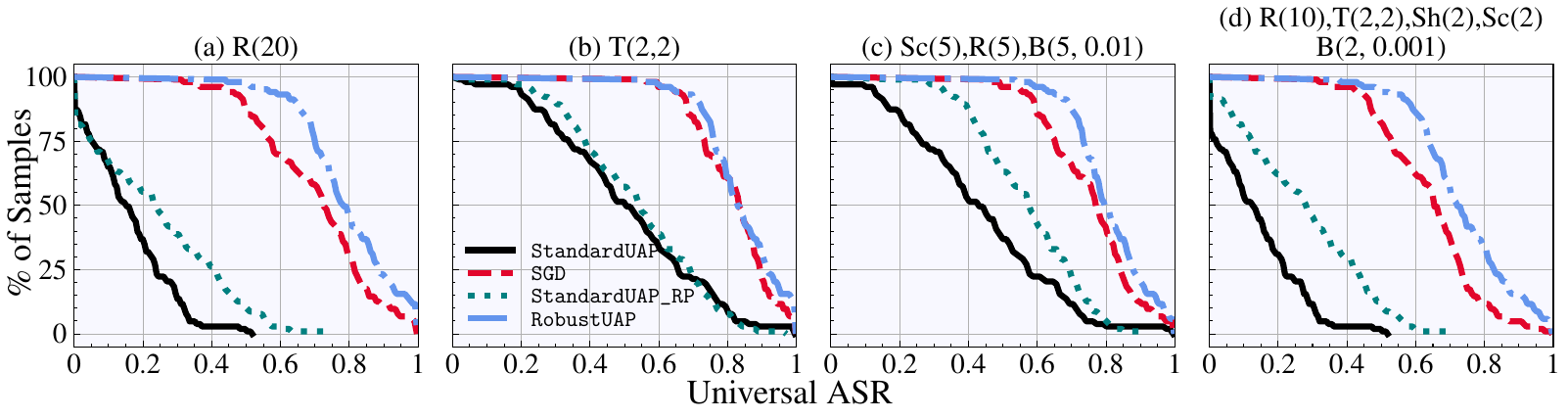}
     \vskip -0.1in
    \caption{For each method, a point $(x,y)$ in the corresponding line represents the percentage of sampled UAPs ($y\%$) with Universal ASR $>x$ for the different semantic transformations on ILSVRC.}
    \label{fig:semantic_top}
\vskip -0.25in
\end{figure*}

\vspace{-0.15in}
\subsection{Robustness to Semantic Transformations}
\vspace{-0.1in}
Next, we consider transformation sets generated by composing five popular semantic transformations in existing literature~\citep{turtle, deepg}: brightness/contrast, rotation, scaling, shearing, and translation.

We use a variety of different compositions to show that our algorithm works under different conditions, and base our parameters for the transformations on~\citep{deepg}. For our experiments, $R(\theta)$ corresponds to rotations with angles between $\pm\theta$; $T(x,y)$, to translations of $\pm x$ horizontally and $\pm y$ vertically; $Sc(p)$ to scaling the image between $\pm p\%$; $Sh(m)$ to shearing by shearing factor between $\pm m\%$; and $B(\alpha, \beta)$ to changes in contrast between $\pm \alpha\%$ and brightness between $\pm \beta$. Further details about these transformations can be seen in Appendix \ref{appendix:transformation}. We consider compositions of different subsets and ranges of these transformations shown in Table~\ref{table:rasr} including composing all transformations together. The hardness of generating robust UAPs depends on the effect that the transformation set has on the UAP (i.e. random noise has a relatively small effect compared to rotation). The hardness also increases with the number of transformations in the composition as well as the range of parameters for each individual transformation. For example, generating robust UAPs is harder for the composition shown in the first and last row for ILSVRC 2012 in Table~\ref{table:rasr} compared to the second and third row. The same is true for generating a UAP robust to uniform random noise.  


\addtolength{\tabcolsep}{-1.5pt}   
\begin{table*}[hbp]
\vskip -0.0in
\caption{Robust ASR of \texttt{RobustUAP} compared to the three baselines.}
\vskip -0.1in
\begin{center}
\begin{footnotesize}
\begin{sc}
\begin{tabular}{llcccc}
\toprule
\multirow{2}{*}{Dataset} & \multirow{2}{*}{Transformation Set} & \texttt{Standard} & \texttt{SGD} & \texttt{Standard} & \texttt{Robust} \\
& & \texttt{UAP} &  & \texttt{UAP\_RP} & \texttt{UAP} \\
\midrule
                             &$R(20)$                                    & $0.0\%$ & $69.9\%$ & $2.9\%$ & $\mathbf{93.2}\%$\\
ILSVRC                       &$T(2,2)$                                   & $35.9\%$ & $96.1\%$ & $38.8\%$ & $\mathbf{97.1}\%$\\
2012                         &$Sc(5), R(5), B(5, 0.01)$                  & $22.3\%$ & $85.4\%$ & $43.7\%$ & $\mathbf{96.1}\%$\\
                             &$R(10), T(2,2), Sh(2), Sc(2), B(2, 0.001)$ & $0.0\%$ & $63.1\%$ & $2.9\%$ & $\mathbf{86.4}\%$\\
\midrule
\multirow{3}{*}{CIFAR-10}    &$R(30), B(2, 0.001)$                       & $0.0\%$ & $64.1\%$ & $2.9\%$ & $\mathbf{75.7}\%$\\
                             &$R(2), Sh(2)$                              & $42.7\%$ & $88.3\%$ & $52.4\%$ & $\mathbf{96.1}\%$\\
                             &$R(10), T(2,2), Sh(2), Sc(2), B(2, 0.001)$ & $0.0\%$ & $58.3\%$ & $7.8\%$ & $\mathbf{79.6}\%$\\
\bottomrule
\end{tabular}
\label{table:rasr}
\end{sc}
\end{footnotesize}
\end{center}
\vskip -0.1in
\end{table*}
\addtolength{\tabcolsep}{1.5pt}

\textbf{Robust ASR} ($\textbf{ASR}_\mathbf{R}$). Figure \ref{fig:semantic_top} shows performance of UAPs obtained by applying 738 randomly sampled transformations to the original UAPs generated by different methods on ILSVRC, similar graphs for CIFAR-10 can be found in Appendix \ref{appendix:diffcdfs}. The \texttt{RobustUAP} algorithm outperforms all others in each case, we observe that for these harder transformation sets \texttt{StandardUAP} loses its effectiveness completely. In Table \ref{table:rasr} we compare robust universal adversarial success rate $\text{ASR}_R$ with $\gamma = 0.6$, in other words, we are finding the percentage of sampled neighbors of the perturbation that are still UAPs with $60\%$ effectiveness on the testing set. We provide average $\text{ASR}_U$ scores as well as $\text{ASR}_R$ for different $\gamma$ levels in Appendix \ref{appendix:diffmetrics}. 

Our \texttt{RobustUAP} algorithm achieves at least $53.4\%$ higher robust ASR when compared to the standard UAP algorithm on both datasets and all transformation sets. Furthermore, our \texttt{RobustUAP} algorithm significantly outperforms both robust baseline approaches. Except for the $T(2,2)$ case which we observe to be the easiest, \texttt{RobustUAP} achieves at least $11.6\%$ performance gain over the baselines. \texttt{SGD} is the best performing baseline and achieves high robust ASR on relatively easier transformation sets performing within $1\%$ of \texttt{RobustUAP} on $T(2,2)$. On harder transformation sets, this gap widens considerably, see Table~\ref{table:rasr}. 

\vspace{-0.15in}
\subsection{Universally Robust UAPs} 
\vspace{-0.1in}

Using the set of primitive transformations discussed by PRIME, we generate robust UAPs on ILSVRC using the same parameters as above. For each sampled transform, we randomly apply three transformations from identity, spectral, spatial, and color. This means that we can get multiple of the same transformation or even no transformation. We follow the setup from PRIME for the parameters of each transformation type. Table \ref{table:urasr} shows the robust ASR when training a \texttt{RobustUAP} on PRIME, Affine ($R(10), T(2,2), Sh(2), Sc(2), B(2, 0.001)$), and Fog transformation sets and how they perform on common corruptions \cite{prime, 3dcorrupt, imagenetc}. Although prime does not inherently contain any specific affine or common corruption in its training it has generally high robustness (>58.3\%) against all transformation sets tested. We observe that training on the target transformation set does bring higher robustness than training on PRIME (i.e. Affine-trained robust UAP has best performance on Gaussian, Contrast, Affine while Fog-trained robust UAP has best performance on fog); however, we find that PRIME has much better performance on unseen transformations (i.e. Fog-trained or Affine-trained robust UAP on JPEG). Our results suggest that a set of good primitive transformations is sufficient for generating universally robust UAPs that generalize well to unseen transformations.

\addtolength{\tabcolsep}{-4.5pt}   
\begin{table*}[t]
\vskip -0.1in
\caption{Robust ASR (\%) of \texttt{RobustUAP} trained on PRIME, Affine ($R(10)$, $T(2,2)$, $Sh(2)$, $Sc(2)$, $B(2, 0.001)$), and Fog when applied to Prime, Affine, and common corruption transforms}
\vskip -0.1in
\begin{center}
\begin{scriptsize}
\begin{sc}
\begin{tabular}{c|cc|ccc|cccc|cccc|cccc}
\toprule
 & \multicolumn{17}{c}{Evaluation Corruption Set}\\
\cmidrule{2-18}
Train & \multicolumn{2}{c}{} & \multicolumn{3}{c}{Noise} & \multicolumn{4}{c}{Blur} & \multicolumn{4}{c}{Weather} & \multicolumn{4}{c}{Digital} \\
Set & Prime & Aff. & Gaus. & Shot & Imp. & Defo. & Glass & Moti. & Zoom & Snow & Fog & Frost & Bright & Contr. & Elast. & Pixel & JPEG \\
\midrule
PRIME   & \textbf{68.4} & 58.3   & 72.1 & 81.3 & \textbf{88.6}            & \textbf{66.5} & \textbf{75.2} & \textbf{81.0} & 74.6  & \textbf{77.8} & 78.8 & \textbf{65.3} & \textbf{85.3}   & 90.4 & \textbf{74.2} & \textbf{69.2} & \textbf{73.3} \\
Affine  & 10.1 & \textbf{86.4}   & \textbf{91.2} & \textbf{93.2} & 85.4   & 45.1 & 31.4 & 76.1 & \textbf{92.4}                    & 65.2 & 70.1 & 50.1 & 80.1                              & \textbf{94.1} & 39.1 & 30.5 & 37.3 \\
Fog     & 1.5  & 0.1             & 10.2 & 11.3 & 9.3                      & 15.1 & 7.6  & 10.1 & 5.5                              & 69.5 & \textbf{95.2} & 21.3 & 10.1                     & 12.6 & 18.4 & 2.8  & 3.9  \\
\bottomrule
\end{tabular}
\label{table:urasr}
\end{sc}
\end{scriptsize}
\end{center}
\vskip -0.1in
\end{table*}
\addtolength{\tabcolsep}{4.5pt}

\begin{figure*}[bh]
\vskip -0.1in
\begin{center}
\centerline{\includegraphics[width=\linewidth]{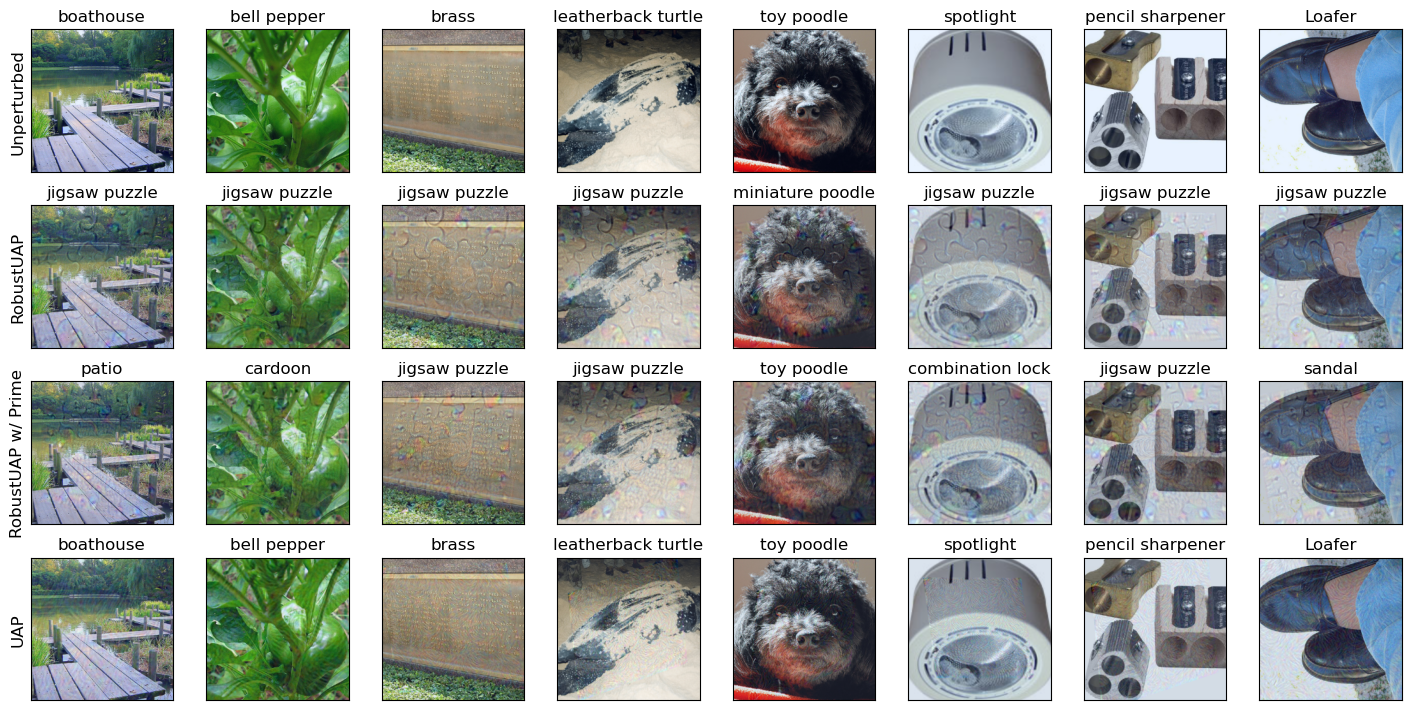}}\vskip -0.1in
\caption{Examples of perturbed images with labels. The top row is unperturbed ILSVRC 2012 test set images, the second row has a randomly transformed robust UAP added to it, the third row has a randomly transformed robust UAP trained with Prime added to it, and the bottom row has a randomly transformed standard UAP added to it. Labels calculated using Inception-v3.}
\label{fig:classification}
\end{center}
\vskip -0.3in
\end{figure*}

\vspace{-0.15in}
\subsection{Visualization} 
\vspace{-0.1in}

We visualize UAPs generated with different algorithms transformed randomly from $R(10)$, $T(2,2)$, $Sh(2)$, $Sc(2), B(2, 0.001)$ and added to images in ILSVRC 2012 in Figure \ref{fig:classification}. Our robust UAPs have a similar level of imperceptibility to standard UAPs. Robust UAPs affect the model classification after transformation with high probability, unlike standard UAPs. 
We further visualize UAPs generated with our three robust algorithms on the same transformation set and dataset in Appendix \ref{appendix:vis}.

\vspace{-0.15in}
\subsection{Transferability of Robust UAPs} 
\vspace{-0.1in}

We evaluate the transferability of \texttt{RobustUAP}. Previous works on UAPs~\citep{uap} show that UAPs are transferable across different models. Here, we will evaluate whether robust UAPs exhibit the same behavior for robustness. The robust UAPs studied here are generated with \texttt{RobustUAP} on $R(10), T(2,2), Sh(2), Sc(2), B(2, 0.001)$ for ILSVRC-2012 with $\gamma = 0.6$. We use a variety of models: Inception-v3~\citep{inception}, MobileNet~\citep{mobilenet}, Inception-v3 trained to be robust on $R(20)$ (InceptionR20), Inception-v3 trained to be robust on horizontal flips (InceptionHF), and ViT~\citep{vit}. Table \ref{table:transferability} shows us that our robust UAPs are transferable between different architectures. Our results show that robust UAPs transfer their robustness properties between architectures and models.  Ignoring ViT, on all of the Inception and MobileNet models, the generated UAPs maintain at least 65\% robust ASR when transferred to each other. This transfer is less but still significant for ViT where it maintains at least 32\% robustness when transferred to or from the other models.

\addtolength{\tabcolsep}{-1.5pt}   
\begin{table*}[ht]
\vskip -0.1in
\caption{Robust ASR when UAP is learned on source model and transfered to target model.}
\begin{center}
\begin{footnotesize}
\begin{sc}
\begin{tabular}{l|ccccc}
\toprule
& \multicolumn{5}{c}{Target Model}\\
\midrule
Source Model & Inception & MobileNet & InceptionR20 & InceptionHF & ViT\\
\midrule
Inception    & $\mathbf{86.4\%}$ & $65.2\%$ & $75.2\%$ & $78.5\%$ & $35.1\%$\\
MobileNet    & $74.3\%$ & $\mathbf{86.2\%}$ & $67.3\%$ & $68.6\%$ & $38.3\%$\\
InceptionR20 & $80.1\%$ & $67.3\%$ & $\mathbf{81.3\%}$ & $73.1\%$ & $32.0\%$\\
InceptionHF  & $77.8\%$ & $70.9\%$ & $75.8\%$ & $\mathbf{83.8\%}$ & $34.6\%$\\
ViT          & $41.2\%$ & $32.4\%$ & $43.2\%$ & $39.7\%$ & $\mathbf{88.5\%}$\\

\bottomrule
\end{tabular}
\label{table:transferability}
\end{sc}
\end{footnotesize}
\end{center}
\vskip -0.175in
\end{table*}
\addtolength{\tabcolsep}{1.5pt}

\vspace{-0.15in}
\subsection{Robustness against Robust UAPs}
\vspace{-0.1in}

Traditional methods for robustness, such as adversarial training, focus on being robust in scenarios where the attacker is powerful (i.e. PGD), but with this comes a significant tradeoff in accuracy. In a preliminary study of practical robustness, we train for robustness against practical attacks such as Robust UAP, while maintaining high accuracy and faster training times. We perform our experiments on CIFAR-10, with a VGG16 model architecture, and with our most challenging transformation set ($R(10), T(2,2), Sh(2), Sc(2), B(2, 0.001)$). We use a batch-wise variant of our robust UAP algorithm to do adversarial training. With this training method, our model obtains a Robust ASR of 0.4\% and an accuracy of 90.1\%. In contrast, standard adversarial training obtains a Robust ASR of 0.2\% and an accuracy of 83.5\%. Here, we can see that our training method obtains almost the same practical robustness without a significant reduction in accuracy. Further, our adversarial training method with robust UAP takes 12 minutes while standard adversarial training takes 48 minutes, which is one indication that our proposed training method is more efficient. We believe that further study can find improvement in robustness, accuracy, and efficiency of this type of training method which allows us to be practically robust while sacrificing less accuracy.

\vspace{-0.15in}
\subsection{Additional Experiments}
\vspace{-0.1in}

In Appendix \ref{appendix:asru} we show how our robust UAPs compare to standard UAPs on the non-robust universal ASR metric. In Appendix \ref{appendix:addmodels}, we evaluate our methods on ResNet18~\citep{resnet} and MobileNet~\citep{mobilenet} for CIFAR-10 and ILSVRC 2012 respectively. The results follow the same trends as those reported in Table \ref{table:rasr}. To address additional real-world transformations, we investigate fog perturbations from \cite{3dcorrupt} in Appendix \ref{appendix:commoncorruptions} finding similar results. 
We show that our methods work well in the targeted attack domain as well, results can be see in Appendix \ref{appendix:targeted}. In Appendix \ref{appendix:dataefficiency}, we show that \texttt{RobustUAP} has good data efficiency and obtains good performance with less data. In Appendix \ref{appendix:sgdvsrobust}, we find that \texttt{RobustUAP} beats out \texttt{SGD} even at the same runtime. The recent popularity of transformer based models has also led us to show that our methods work on transformer based networks, results in Appendix \ref{appendix:transformers}. In Appendix \ref{appendix:robustnetwork}, we find that traditional model robustness does not seem to effect ability to create robust UAPs. Finally, in Appendix \ref{appendix:ablationopt}, we perform an ablation study on optimization strategy and show that SGD outperforms other popular optimizers.\vspace{-0.15in}
\section{Related Work}\vspace{-0.05in}


\textbf{UAP Algorithms. }Most works focusing on UAPs~\citep{uap, gduap, mutualuap, singularuap, uapdefense, uap_segmentation, zhang2020cd} generate singular vectors and do not consider perturbation robustness.~\citet{wireless_uap} introduces a perturbation generator model (PGM) for the wireless domain which creates UAPs. They show that both adversarial training and noise subtracting defenses used in the wireless domain are highly effective in mitigating the effects of a single vector UAP attack; they further show that their method of generating a set of UAPs is an effective way for an attacker to circumvent these defenses. Although PGM provides a method for efficiently sampling unique UAPs, it is not robust to real-world transformations. In contrast, our method enables efficient sampling of UAPs that are robust to transformations.    


\textbf{Robust Adversarial Examples. }The following papers introduce notions of robustness under different viewpoints and environmental conditions for constructing realizable adversarial examples. This is a different threat model compared to the additive perturbations discussed in this paper. \citet{human_jpeg} constructs adversarial examples which minimize human detectability, further introducing the idea of robustness for adversarial examples. They show that their attacks are robust against jpeg compression. \citet{glasses} attack facial recognition systems by putting adversarial perturbations on glass frames. Their work demonstrates a successful physical attack under stable conditions and poses. \citet{stopsign} proposes Robust Physical Perturbations ($\text{RP}_2$) in order to show that adding graffiti on a stop sign can cause it to be misclassified in both simulations and in the real world. \citet{turtle} introduce Expectation over Transformation (EOT) and use it to print real-world objects which are adversarial given a range of physical and environmental conditions.

\textbf{Robust Adversarial Perturbations. }\citet{music:19} generates music which affects a voice assistant based system from picking up its wake word. \citet{sticker:19} presents a method for generating a targeted adversarial sticker which changes an image classifier's classification from one pre-specified class to another. Both of these methods rely on specific use cases and are tailored towards generating adversaries coming from strict distributions, e.g. \citep{music:19} generates guitar music while \citep{sticker:19} generates a small grid of dots. These works build on algorithms akin to our baseline approaches and are limited in scope to domain specific transformations. Our work provides a framework for improving robustness against a wide range of transformations in diverse domains and can be leveraged for improving the effectiveness of these attacks.\vspace{-0.15in}

\section{Limitations and Ethics}
\vspace{-0.1in}
We have shown the benefits of \texttt{RobustUAP} and would like to outline a few limitations and ethical concerns. Firstly, we note that our methods do not have a way to address non-differentiable transformations. We hope that future work leverges methods such as REINFORCE which do not have dependence on differentiability \cite{reinforce}. Secondly, \texttt{RobustUAP} would not work against models trained with standard adversarial training since to have a UAP you need to be able to generate standard adversarial examples. However, currently, robustly trained models are seldom used since adversarial attacks are hard to realize (i.e. UAPs) and these models come with a large tradeoff with accuracy. Our preliminary research suggests that defending against practical attacks such as robust UAP does not come with the same tradeoff, allowing for more practical robustness while retaining similar accuracy. We understand that our proposed methods could cause harm to existing deployed ML methods. Our hope is to expose vulnerabilities in existing safety and security critical applications of ML to spawn further research into practical robustness against real-world implementable attacks.
\vspace{-0.1in}
\section{Conclusion}\vspace{-0.05in}

In this paper, 
we demonstrate that standard UAPs fail to be universally adversarial under transformation. We propose a new method, \texttt{RobustUAP}, to generate robust UAPs based upon obtaining probabilistic bounds on UAP robustness across an entire transformation space. We show that \texttt{RobustUAP} works for both known and unknown transformation sets. Our experiments provide empirical evidence that our principled approach generates UAPs that are more robust than those from the existing/baseline methods. Our preliminary work suggests that robustness against practical adversaries such as robust UAPs may require much less tradeoff with accuracy and we hope that inspires research into robustness against practical attacks. 


\bibliography{paper}
\bibliographystyle{plainnat}

\newpage
\appendix
\onecolumn
\section*{Appendix}
\section{Definitions}\label{appendix:definitions}

In this section, we will formally define the terms used in the main body of the paper. We first start with adversarial examples.

\begin{definition}
\label{def:ae}
Given a correctly classified point $\mathbf{x}$, a distance function $d(\cdot, \cdot):\mathbb{R}^d\times\mathbb{R}^d\to\mathbb{R}$, and bound $\epsilon \in \mathbb{R}$,  $\mathbf{x'}$ is an \textit{adversarial example} iff $d(\mathbf{x'}, \mathbf{x}) < \epsilon$ and $\hat{f}(\mathbf{x'})\neq y$.
\end{definition}

We distinguish between adversarial examples and perturbations. An \textit{adversarial perturbation} added to the point it is attacking is an \textit{adversarial example}, $x' = x + v$. We can construct $v$ by solving the following optimization problem:

\begin{equation}
    \argmin_{v} ||v||_p \text{ s.t. } \hat{f}(x + v) \neq \hat{f}(x)
\end{equation}

Here, we are looking for the smallest $v$ such that $f$'s classification changes from the original output (assuming $f$ correctly classified $x$).

Next, we define the adversarial success rate which measures whether or not a given perturbation is adversarial. 

\begin{definition}
\label{def:asr}
Given a datapoint $x$, and perturbation $v$, \textit{adversarial success}, $\text{AS}$, is defined as
\begin{equation}
    \text{AS}(f, x, v) = 1 - \delta(\hat{f}(x + v), \hat{f}(x))
\end{equation}

Here, $\delta(i, j)$ refers to the Kronecker Delta function~\citep{calc}, formally,

\begin{equation}
    \delta(i, j) = \begin{cases}0 &\text{if } i \neq j\\1 &\text{if } i = j\end{cases}
\end{equation}

With $\text{AS}$, we measure whether $v$ changes $f$'s classification of $x$.
\end{definition}

Using the definision of adversarial success we can now define unniversal adversarial success rate.

\begin{definition}
\label{def:asru}
Given a data distribution $\mu$, and perturbation $\mathbf{u}$, \textit{universal adversarial success rate}, $\text{ASR}_U$, for $\mathbf{u}$, is

\begin{equation}
    \text{ASR}_U(f, \mu, \mathbf{u}) = \mathop{P}_{\mathbf{x}\sim \mu}(\hat{f}(\mathbf{x}+\mathbf{u}) \neq \hat{f}(\mathbf{x}))
\end{equation}
\end{definition}

Using Definition \ref{def:asru}, we formally define a UAP.

\begin{definition}
\label{def:uap}
A \textit{universal adversarial perturbation} is a vector $\mathbf{u} \in \mathbb{R}^d$ which, when added to almost all datapoints in $\mu$ causes the classifier $f$ to misclassify. Formally, given $\gamma$, a bound on universal ASR, and $l_p$-norm with corresponding bound $\epsilon$, $\mathbf{u}$ is a UAP iff $\text{ASR}_U(f, \mu, \mathbf{u}) > \gamma$ and $||\mathbf{u}||_p < \epsilon$.

\end{definition}

Now, we move onto definitions pertaining to robust UAPs. We start by formally defining transformation sets and neighborhoods.

\begin{definition}
\label{def:neighborhood}
A \textit{transformation}, $\tau$, is a composition of bijective sub-differentiable transformation functions. A \textit{transformation set}, $T$, is a set of distinct transformations. A point $\mathbf{v'}$ is in the \textit{neighborhood} $N_T(\mathbf{v})$, of $\mathbf{v}$, if there is a transform in $T$ that maps $\mathbf{v}$ to $\mathbf{v'}$. Formally, 
\begin{equation}
    \mathbf{v'} \in N_T(\mathbf{v}) \iff \exists \tau \in T \text{ s.t. } \tau(\mathbf{v}) = \mathbf{v'}
\end{equation}
\vspace{-0.15in}
\end{definition}

Finally, we formally define robust universal adversarial success rate.

\begin{definition}
\label{def:asrr}
Given a data distribution $\mu$, transformation set $T$, universal ASR level $\gamma$, bound $\epsilon$ on $l_p$-norm, and perturbation $\mathbf{u_r}$, \textit{robust universal adversarial success rate}, $\text{ASR}_R$, is defined as,

\begin{equation}
        \text{AS}\text{R}_R(f, \mu, T, \gamma, \mathbf{u_r}) = \mathop{P}_{\mathbf{u_r'}\sim N_T(\mathbf{u_r})}(\text{ASR}_U(f, \mu, \mathbf{u_r'}) > \gamma \land ||\mathbf{u_r'}||_p < \epsilon)
\end{equation}

\end{definition}
\section{Proof of Theorem \ref{thrm:robust}}\label{appendix:theorem}

This proof relies heavily on similar proofs provided by \citet{chernoff1952measure} and by \citet{alippi}. We refer to the reader to these texts for further details. In our proof, we show how to adapt universal ASR to these proofs.

\begin{proof} 
The bound on $n$ is derived via the Chernoff inequality applied to $\hat{p}_n(\gamma)$ and $\mathop{\mathbb{E}}[\hat{p}_n(\gamma)] = p(\gamma)$~\citep{chernoff1952measure, alippi}. Equation~\ref{eq:theorem_eq} holds since computing universal ASR is Lebesgue measurable over the data distribution and since we assume $N_T(\mathbf{u_r})$ has a well defined PDF. 
\end{proof}
\section{Semantic Transformations}
\label{appendix:transformation}
In this section, we discuss the semantic transformations used in the paper. Brightness and contrast can be represented via \textit{bias} ($\beta$) and \textit{gain} ($\alpha > 0$) parameters respectively. Formally, if $\mathbf{x}$ is the original image, then the transformed image, $\mathbf{x'}$, can be represented as

\begin{equation}
    \mathbf{x'\mathbf} = \alpha \mathbf{x} + \beta
\end{equation}

Rotation, scaling, shearing, and translation are all affine transformations acting on the coordinate system, $c$, of the images instead of the pixel values, $\mathbf{x}$. In order to recover the pixel values and differentiate over the transformation, we will need sub-differentiable interpolation, see Appendix \ref{appendix:interpolation}. For finite dimensions, affine transformations can be represented as a linear coordinate map where the original coordinates are multiplied by an invertible augmented matrix and then translated with additional bias vector. Below, we give the general form for an affine transformation given augmented matrix $\mathbf{A}$, bias matrix $\mathbf{b}$, and input coordinates $c$. We can compute the output coordinates, $c'$, as 

\begin{equation}
    \begin{bmatrix}\mathbf{c'} \\ 1\end{bmatrix} = \begin{bmatrix}[ccc|c]&\mathbf{A}&&\mathbf{b}\\0&\dots&0&1  \end{bmatrix}\begin{bmatrix}\mathbf{c}\\1\end{bmatrix}
\end{equation}

Below, we give the augmented matrix $\mathbf{A}$ and additional bias matrix $\mathbf{b}$ for rotation, scaling, shearing, and translation.

Rotation, $R(\theta)$, by $\theta$ degrees:
\begin{equation}
    \mathbf{A} = \begin{pmatrix}\cos \theta & -\sin\theta\\\sin\theta&\cos\theta\end{pmatrix} \text{, } \mathbf{b} = \begin{pmatrix}0\\0\end{pmatrix}
\end{equation}

Scaling, $Sc(p)$, by $p\%$:
\begin{equation}
    \mathbf{A} = \begin{pmatrix} 1 + \frac{p}{100}& 0\\0&1 + \frac{p}{100}\end{pmatrix} \text{, } \mathbf{b} = \begin{pmatrix}0\\0\end{pmatrix}
\end{equation}

Shearing, $Sh(m)$, by shear factor $m\%$:
\begin{equation}
    \mathbf{A} = \begin{pmatrix}1& 1 + \frac{m}{100}\\0&1\end{pmatrix} \text{, } \mathbf{b} = \begin{pmatrix}0\\0\end{pmatrix}
\end{equation}

Translation, $T(x, y)$, by $x$ pixels horizontally and $y$ pixels vertically:
\begin{equation}
    \mathbf{A} = \begin{pmatrix}0&0\\0&0\end{pmatrix} \text{, } \mathbf{b} = \begin{pmatrix}x\\y\end{pmatrix}
\end{equation}

\section{Interpolation} 
\label{appendix:interpolation}
Affine transformations may change a pixel's integer coordinates into non-integer coordinates. Interpolation is typically used to ensure that the resulting image can be represented on a lattice (integer) pixel grid. For this paper, we will be using bilinear interpolation, a common interpolation method which achieves a good trade-off between accuracy and efficiency in practice and is commonly used in literature~\citep{stae, deepg}. Let $x_{i,j}$, $x_{i,j}'$ represent the pixel value at position $i, j$ for the original and transformed image respectively. Let ${c'}_{i,j}^x$, ${c'}_{i,j}^y$ represent the $x$-coordinate and $y$-coordinate of the pixel at $i,j$ after transformation. We define our transformed image by summing over all pixels $n,m \in [1\dots H]\times[1\dots W]$ where $H$ and $W$ represent the height and width of the image.

\begin{equation}
        x_{i,j}' = \sum_n^H\sum_m^W x_{n,m}\max(0, 1-|{c'}_{i,j}^x-m|)\max(0, 1- |{c'}_{i,j}^y-n|)
\end{equation}

This interpolation can be computed for each channel in the image. While interpolation is typically not differentiable, in order to generate adversarial examples using standard techniques we need a differentiable version of interpolation.~\citep{stn} introduces differentiable image sampling.
Their method works for any interpolation method as long as the (sub-)gradients can be defined with respect to $x, {c'}_{i,j}$. For bilinear interpolation this becomes, 
\begin{equation}
        \frac{\partial x'_{i,j}}{\partial x_{n,m}} = \sum_n^H\sum_m^W\max(0, 1- |{c'}_{i,j}^x-m|)\max(0, 1- |{c'}_{i,j}^y-n|)
\end{equation}

\begin{equation}
        \frac{\partial x'_{i,j}}{\partial {c'}_{i,j}^x} = \sum_n^H\sum_m^W x_{n,m}\max(0, 1- |{c'}_{i,j}^y-n|)
        \begin{cases}
            1&\text{if } m\geq |{c'}_{i,j}^x-m|\\
            -1 &\text{if } m< |{c'}_{i,j}^x-m|\\
            0&\text{otherwise}
        \end{cases}
\end{equation}
\section{Baseline Algorithms}\label{appendix:baseline}

\subsection{Stochastic Gradient Descent}

The first baseline directly solves Equation \ref{eq:ruap_opt} using gradient descent. Since we are solving a constrained optimization problem, we cannot use gradient descent directly. Instead, we can solve the Lagrangian-relaxed form of the problem (adding a weighted norm term to the minimization problem), as in~\citep{cw, turtle}, using a momentum based Stochastic Gradient Descent (SGD). ~\citet{uat} suggests that this is an effective method for generating standard UAPs. Our SGD algorithm is in Appendix \ref{appendix:sgdalg}. In order to implement it, we replace the Delta function with a loss function, $L$. We iteratively converge towards the inner expectation by computing it in batches, and towards the outer expectation by sampling a large number of transformations. Given that we would like to estimate on a batch, $\mathbf{\hat{x}} \subset \mu$, and a random set of transformations sampled from $T$, $\hat{\tau} \subset T$, we can approximate using:

\vspace{-0.15in}
\begin{equation}
     \frac{I(||\hat{\tau}_j(\mathbf{u_r})||<\epsilon)}{|\mathbf{\hat{x}}|\times |\hat{\tau}|} \sum_{i = 1}^{|\mathbf{\hat{x}}|}\sum_{j = 1}^{|\hat{\tau}|} L[f(\mathbf{\hat{x}_i} + \hat{\tau}_j(\mathbf{u_r})), f(\mathbf{\hat{x}_i)}] - \lambda ||\mathbf{u_r}||_p
\end{equation}
\vspace{-0.1in}

\subsection{Standard UAP Algorithm with Robust Adversarial Perturbations}
For our second baseline, we leverage the standard UAP algorithm from~\citet{uap} (see Appendix \ref{appendix:iuapalg} for the algorithm). The standard UAP algorithm takes each input, $\mathbf{x_i}$, computes the smallest additive change, $\Delta \mathbf{u}$, to the current perturbation, $\mathbf{u}$, that would make $\mathbf{u}+\Delta \mathbf{u}$ an adversarial perturbation for $\mathbf{x_i}$. Intuitively, over time the algorithm will approach a perturbation that works on most inputs in the training dataset. We modify this approach by computing robust adversarial perturbations rather than standard adversarial perturbations. At each point $\mathbf{x_i}$, we compute the smallest additive change, $\Delta \mathbf{u_r}$, to the current robust adversarial perturbation, $\mathbf{u_r}$, that would make $\mathbf{u_r} + \Delta \mathbf{u_r}$ a robust adversarial perturbation for $\mathbf{x_i}$. We search for robust adversarial perturbations by optimizing the expectation that a point in the neighborhood of $\mathbf{v_r}$ is adversarial while restricting the perturbation to an $l_p$ norm of $\epsilon$.

\section{SGD Algorithm}\label{appendix:sgdalg}

Our SGD UAP algorithm is based on standard momentum based SGD while optimizing over the objective proposed in \ref{eq:ruap_opt}, the algorithm details can be seen in Algorithm \ref{alg:sgd_uap}.

\begin{algorithm}[ht]
   \caption{Stochastic Gradient Descent UAP Algorithm}
   \label{alg:sgd_uap}
\begin{algorithmic}[1]
   \STATE Initialize $\mathbf{u_r} \gets 0, \Delta \mathbf{u_r} \gets 0$
   \REPEAT
   \FOR{$\mathbf{B} \in \mathbf{X}$}
   \STATE Sample $\hat{t} \subset T$
   \STATE $\Delta \mathbf{u_r} \gets \alpha \Delta \mathbf{u_r} - \frac{\nu}{|\mathbf{\hat{x}}|\times |\hat{t}|} \sum_{i = 1}^{|\mathbf{\hat{x}}|}\sum_{j = 1}^{|\hat{t}|} \nabla L[f(\mathbf{\hat{x}_i} + \hat{t}_j(\mathbf{u_r})), f(\mathbf{\hat{x}_i})]$
   \STATE Update the perturbation with projection:
   \STATE $\mathbf{u} \gets \mathcal{P}_{p, \epsilon}(\mathbf{u_r} + \Delta \mathbf{u_r})$
   \ENDFOR
   \UNTIL{$ASR_R(f, \mathbf{X}, T, \gamma, \mathbf{u_r}) < \zeta$}
\end{algorithmic}
\end{algorithm}

\section{Iterative UAP Algorithm} \label{appendix:iuapalg}

\citet{uap} introduces an iterative UAP algorithm, the algorithm can be seen in Algorithm~\ref{alg:uap}.

\begin{algorithm}[ht]
   \caption{Iterative Universal Perturbation Algorithm (\citet{uap})}
   \label{alg:uap}
\begin{algorithmic}[1]
   \STATE Initialize $\mathbf{u} \gets 0$
   \REPEAT
   \FOR{$\mathbf{x_i} \in \mathbf{X}$}
   \IF{$\hat{f}(\mathbf{x_i} + \mathbf{u}) = \hat{f}(\mathbf{x_i})$}
   \STATE Compute minimal adversarial perturbation: 
   \STATE $\Delta \mathbf{u} \gets \text{arg}\min_{\mathbf{r}} ||\mathbf{r}||_2$ s.t. $\hat{f}(\mathbf{x_i} + \mathbf{u} + \mathbf{r}) \neq \hat{f}(\mathbf{x_i})$
   \STATE Update the perturbation with projection:
   \STATE $\mathbf{u} \gets \mathcal{P}_{p, \epsilon}(\mathbf{u} + \Delta \mathbf{u})$
   \ENDIF
   \ENDFOR
   \UNTIL{$ASR_U(f, \mathbf{X}, \mathbf{u}) < \gamma$}
\end{algorithmic}
\end{algorithm}

\section{Estimate Robustness Algorithm} \label{appendix:eralg}

In this section, we give our algorithm for estimating the robustness of a UAP.

\begin{algorithm}[ht]
   \caption{EstRobustness}
   \label{alg:asrr}
\begin{algorithmic}[1]
   \STATE Draw $n = \lceil \frac{1}{2\psi^2}\ln\frac{2}{\phi}\rceil$ samples $\tau_i \sim T$
   \STATE \textbf{Return}
   $ \hat{p}_n(\gamma) = \frac{1}{n} \sum_i^n I(ASR_U(f, \mathbf{X}, \tau_i(\mathbf{u_r})) > \gamma)$
\end{algorithmic}
\end{algorithm}

\section{Experiment Parameters}\label{appendix:parameters}

In our experiments, we have capped all algorithms at 5 epochs or if they have achieved an $\text{ASR}_R$ of $0.95$. The UAPs are trained with the same transformation set that they are evaluated on. For algorithms running PGD internally, we have capped the number of iterations to $40$.
\section{Further Evaluation of Uniform Noise}\label{appendix:noise}

A table of results for uniform random noise can be seen in Table \ref{table:noise}.

\addtolength{\tabcolsep}{-1pt}   
\begin{table*}[ht]
\caption{Robust ASR with uniform random noise, $\gamma = 0.8$.}
\begin{center}
\begin{small}
\begin{sc}
\begin{tabular}{llcccc}
\toprule
\multirow{2}{*}{Dataset} & \multirow{2}{*}{Transformation Set} & \texttt{Standard} & \texttt{SGD} & \texttt{Standard} & \texttt{Robust} \\
& & \texttt{UAP} &  & \texttt{UAP\_RP} & \texttt{UAP} \\
\midrule
ILSVRC                       &$U(0.1)$ & $81.6\%$ & $94.2\%$ & $91.3\%$ & $\mathbf{99.0}\%$\\
2012                         &$U(0.3)$ & $10.7\%$ & $68.9\%$ & $42.7\%$ & $\mathbf{96.1}\%$\\
\midrule
\multirow{2}{*}{CIFAR-10}    &$U(0.1)$ & $66.0\%$ & $98.1\%$ & $96.1\%$ & $\mathbf{100}\%$\\
                             &$U(0.3)$ & $5.8\%$ & $96.1\%$ & $47.6\%$ & $\mathbf{100}\%$\\
\bottomrule
\end{tabular}
\label{table:noise}
\end{sc}
\end{small}
\end{center}
\end{table*}
\addtolength{\tabcolsep}{1pt}

\section{UAP performance against semantic transformations on CIFAR-10}\label{appendix:diffcdfs}

In this section, we show similar results as shown in ILSVRC-2012 but on CIFAR-10. Here, we again see that RobustUAP outperforms all other baselines.

\begin{figure*}[ht]

     \includegraphics[width=\textwidth]{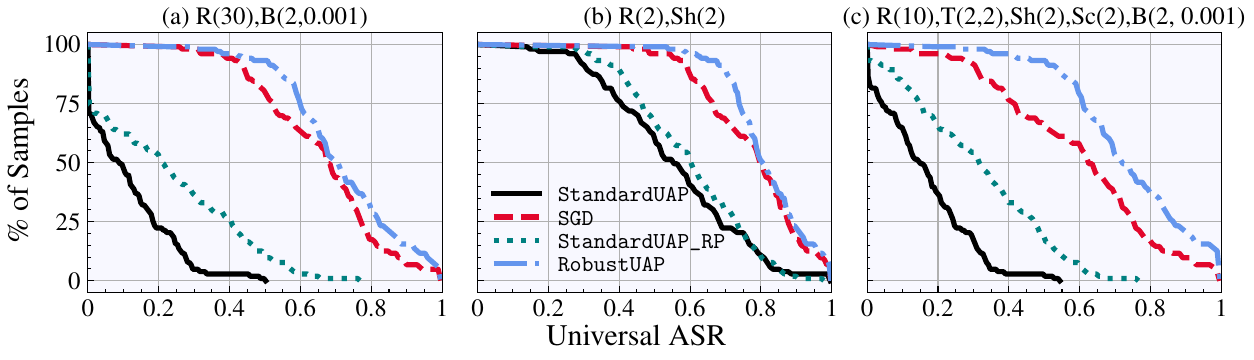}
     
    \caption{For each method, a point $(x,y)$ in the corresponding line represents the percentage of sampled UAPs ($y\%$) with Universal ASR $>x$ for the different semantic transformations on CIFAR-10.}
    \label{fig:semantic_bot}
\end{figure*}

\section{Average \texorpdfstring{$\text{ASR}_U$}{ASRU} and \texorpdfstring{$\text{ASR}_R$}{ASRR} with different \texorpdfstring{$\gamma$}{γ}'s}\label{appendix:diffmetrics}

We provide additional metrics computed on the same set of transformations, datasets, and models as in Table \ref{table:rasr}. In Table \ref{table:a_asru}, we present the Average $\text{ASR}_U$ rather than $\text{ASR}_R$. The average shows us that our \texttt{RobustUAP} algorithm creates UAPs which after transformation on average are better UAPs than all other algorithms. We observe that the average shows us that even standard UAPs aren't completely ineffective after transformation they just have a very low chance of being highly effective.

\addtolength{\tabcolsep}{-1.5pt}   
\begin{table*}[ht]
\caption{Average Universal ASR of our Robust UAP algorithms and the standard UAP~\citep{uap} method.}
\begin{center}
\begin{small}
\begin{sc}
\begin{tabular}{llcccc}
\toprule
\multirow{2}{*}{Dataset} & \multirow{2}{*}{Transformation Set} & \texttt{Standard} & \texttt{SGD} & \texttt{Standard} & \texttt{Robust} \\
& & \texttt{UAP} &  & \texttt{UAP\_RP} & \texttt{UAP} \\
\midrule
                             &$R(20)$                                    & $16.3\%$ & $71.5\%$ & $24.7\%$ & $\mathbf{81.3\%} $\\
ILSVRC                       &$T(2,2)$                                   & $52.6\%$ & $82.6\%$ & $55.4\%$ & $\mathbf{85.4\%} $\\
2012                         &$Sc(5), R(5), B(5, 0.01)$                  & $44.9\%$ & $76.3\%$ & $58.5\%$ & $\mathbf{82.2\%} $\\
                             &$R(10), T(2,2), Sh(2), Sc(2), B(2, 0.001)$ & $13.6\%$ & $64.8\%$ & $29.0\%$ & $\mathbf{75.3\%} $\\
\midrule
\multirow{3}{*}{CIFAR-10}    &$R(30), B(2, 0.001)$                       & $9.9\%$  & $66.8\%$ & $22.2\%$ & $\mathbf{73.4\%} $\\
                             &$R(2), Sh(2)$                              & $57.1\%$ & $78.8\%$ & $61.2\%$ & $\mathbf{82.9\%} $\\
                             &$R(10), T(2,2), Sh(2), Sc(2), B(2, 0.001)$ & $16.2\%$ & $61.2\%$ & $32.6\%$ & $\mathbf{76.4\%} $\\
\bottomrule
\end{tabular}
\label{table:a_asru}
\end{sc}
\end{small}
\end{center}
\end{table*}
\addtolength{\tabcolsep}{1.5pt}

In Table \ref{table:rasr_dg}, we present $\text{ASR}_R$ computed at $\gamma = [0.5, 0.7]$ rather than $\gamma = 0.6$. This table shows a similar story to above, and shows that our algorithm produces better results under a variety of success thresholds.

\addtolength{\tabcolsep}{-4.5pt}   
\begin{table*}[ht]
\caption{Robust ASR of our Robust UAP algorithms and the standard UAP~\citep{uap} method with $\gamma = [0.5, 0.7]$.}
\begin{center}
\begin{small}
\begin{sc}
\begin{tabular}{llcccccccc}
\toprule
\multirow{2}{*}{Dataset} & \multirow{2}{*}{Transformation Set} & \multicolumn{2}{c}{\texttt{Standard}} & \multicolumn{2}{c}{\texttt{SGD}} & \multicolumn{2}{c}{\texttt{Standard}} & \multicolumn{2}{c}{\texttt{Robust}} \\
& & \multicolumn{2}{c}{\texttt{UAP}} & & & \multicolumn{2}{c}{\texttt{UAP\_RP}} & \multicolumn{2}{c}{\texttt{UAP}} \\
& & 0.5 & 0.7 & 0.5 & 0.7 & 0.5 & 0.7 & 0.5 & 0.7\\
\midrule
                             &$R(20)$                                    & $1.9\%$ & $0.0\%$ & $88.3\%$ & $58.3\%$ & $10.7\%$ & $1.0\%$ & $\mathbf{98.1}\%$ & $\mathbf{76.7}\%$\\
ILSVRC                       &$T(2,2)$                                   & $51.5\%$ & $21.4\%$ & $\mathbf{100}\%$ & $84.5\%$ & $57.3\%$ & $23.3\%$ & $\mathbf{100}\%$ & $\mathbf{91.3}\%$\\
2012                         &$Sc(5), R(5), B(5, 0.01)$                  & $38.8\%$ & $11.7\%$ & $96.1\%$ & $67.0\%$ & $64.1\%$ & $25.2\%$ & $\mathbf{99.0}\%$ & $\mathbf{87.4}\%$\\
                             &$R(10), T(2,2), Sh(2), Sc(2), B(2, 0.001)$ & $1.9\%$ & $0.0\%$ & $82.5\%$ & $38.8\%$ & $12.6\%$ & $1.0\%$ & $\mathbf{95.1}\%$ & $\mathbf{59.2}\%$\\
\midrule
\multirow{3}{*}{CIFAR-10}    &$R(30), B(2, 0.001)$                       & $1.0\%$ & $0.0\%$ & $80.6\%$ & $43.7\%$ & $12.6\%$ & $1.0\%$ & $\mathbf{93.2}\%$ & $\mathbf{49.5}\%$\\
                             &$R(2), Sh(2)$                              & $62.1\%$ & $22.3\%$ & $96.1\%$ & $68.9\%$ & $68.0\%$ & $30.1\%$ & $\mathbf{99.0}\%$ & $\mathbf{89.3}\%$\\
                             &$R(10), T(2,2), Sh(2), Sc(2), B(2, 0.001)$ & $2.9\%$ & $0.0\%$ & $67.0\%$ & $38.8\%$ & $19.4\%$ & $1.0\%$ & $\mathbf{93.2}\%$ & $\mathbf{55.3}\%$\\
\bottomrule
\end{tabular}
\label{table:rasr_dg}
\end{sc}
\end{small}
\end{center}
\end{table*}
\addtolength{\tabcolsep}{4.5pt}
\section{Robustness to Random Noise}\label{appendix:randomnoise}

First, we show that our algorithm generates UAPs robust against uniform random noise. Here our neighborhood is defined as an $L_\infty$ ball of radius $\epsilon$ around the perturbation. $U(\epsilon)$ represents noise drawn uniformly from such a ball. Figure \ref{fig:robustnoise_ilsvrc} shows the performance of each algorithm. For example, the \texttt{RobustUAP} algorithm achieves a $\text{ASR}_U$ of 0.9 greater than $97\%$ of the time under $U(0.1)$ on CIFAR-10, where all other algorithms achieve 0.9 at most $30\%$ of the time. \texttt{RobustUAP} outperforms all other algorithms for both noise sizes. \texttt{StandardUAP} has a lower mean and higher variance in universal ASR and is much less robust to transformation. A table of Robust ASR results for $\gamma = 0.8$ can be seen in Appendix \ref{appendix:noise}. Our Robust ASR results are guaranteed to be $\pm 0.05$ from the actual result with a probability of $95\%$. For example, we estimate that \texttt{RobustUAP} has $\text{ASR}_R$ of $96.1\%$ for U(0.3), we are guaranteed that the true robustness is $>91.1\%$ with a probability of $95\%$. Note that we get robustness guarantees from \texttt{EstRobustness} as our neighborhood has a well-defined PDF.

\begin{figure*}[htb]
    \centering
    \includegraphics[width=\textwidth]{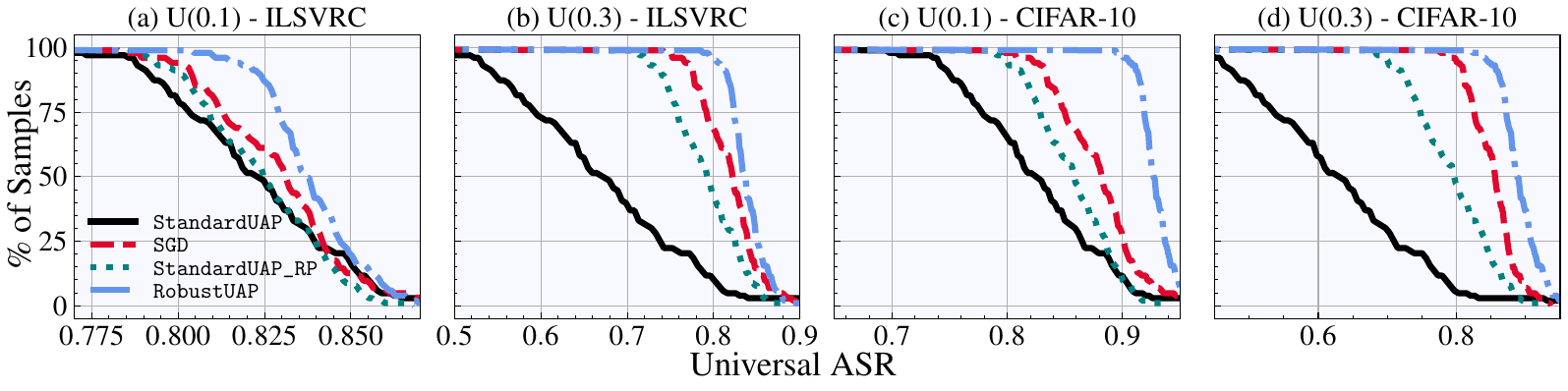}
   \caption{For each method, a point $(x,y)$ in the corresponding line represents the percentage of sampled UAPs ($y\%$) with Universal ASR $>x$ for $U(0.1)$ and $U(0.3)$ on ILSVRC and CIFAR-10.}
   \label{fig:robustnoise_ilsvrc}
\end{figure*}

\section{Comparison on non-Robust Universal ASR metric}\label{appendix:asru}

We compare our robust UAPs to standard UAPs on the non-robust universal ASR metric, see Table \ref{table:uasr}. All robust UAPs are generated to be robust against $R(10), T(2,2), Sh(2), Sc(2), B(2, 0.001)$. We observe that at the same $l_2$-norm all robust UAPs achieve a lower universal ASR than the standard UAP algorithm. This result is not too surprising as solving the optimization problem for robust UAP is significantly more difficult. We further observe that our \texttt{RobustUAP} algorithm is the most effective in comparison to the other robust baseline approaches. 

\begin{table*}[htb]
    \caption{Universal ASR of our Robust UAP algorithms and the standard UAP method.}
    \begin{center}
    \begin{small}
    \begin{sc}
    \begin{tabular}{lcccc}
    \toprule
    Dataset & \texttt{StandardUAP} & \texttt{SGD} & \texttt{StandardUAP\_RP} & \texttt{RobustUAP} \\
    \midrule
    ILSVRC 2012 & $\mathbf{95.5\%}$ & $85.6\%$ & $82.3\%$ & $91.3\% $\\
    \midrule
    CIFAR-10    & $\mathbf{96.2\%}$ & $89.3\%$ & $84.0\%$ & $93.7\% $\\
    \bottomrule
    \end{tabular}
    \label{table:uasr}
    \end{sc}
    \end{small}
    \end{center}
    \end{table*}

\section{Additional Models}\label{appendix:addmodels}

We also provide additional data on our methods evaluated on the same transformations and datasets but on different models. In this case, we use ResNet-18 \citep{resnet} for CIFAR-10 and MobileNet \citep{mobilenet} for ILSVRC 2012. Results can be seen in Table \ref{table:rasr_dm}. We observe similar performance across models suggesting that the performance of the attacks is more directly tied to transformation set and dataset.

\addtolength{\tabcolsep}{-4pt}   
\begin{table*}[htb]
\caption{Robust ASR on Resnet-18 for CIFAR-10 and MobileNet for ILSVRC 2012.}
\begin{center}
\begin{small}
\begin{sc}
\begin{tabular}{lllcccc}
\toprule
\multirow{2}{*}{Dataset} & \multirow{2}{*}{Model} & \multirow{2}{*}{Transformation Set} & \texttt{Standard} & \texttt{SGD} & \texttt{Standard} & \texttt{Robust} \\
& & & \texttt{UAP} &  & \texttt{UAP\_RP} & \texttt{UAP} \\
\midrule
                          & \multirow{4}{*}{MobileNet} &$R(20)$                                    & $8.1\%$ & $71.2\%$ & $2.6\%$ & $\mathbf{85.0\%} $\\
ILSVRC                    &                            &$T(2,2)$                                   & $40.9\%$ & $98.7\%$ & $54.3\%$ & $\mathbf{99.6\%} $\\
2012                      &                            &$Sc(5), R(5), B(5, 0.01)$                  & $16.3\%$ & $94.5\%$ & $44.3\%$ & $\mathbf{96.3\%} $\\
                          &                            &$R(10), T(2,2), Sh(2), Sc(2), B(2, 0.001)$ & $4.1\%$ & $75.7\%$ & $8.6\%$ & $\mathbf{86.2\%} $\\
\midrule
\multirow{3}{*}{CIFAR-10} & \multirow{3}{*}{ResNet-18} &$R(30), B(2, 0.001)$                       & $0.9\%$ & $67.8\%$ & $6.4\%$ & $\mathbf{74.9\%} $\\
                          &                            &$R(2), Sh(2)$                              & $49.9\%$ & $99.5\%$ & $49.1\%$ & $\mathbf{99.8\%} $\\
                          &                            &$R(10), T(2,2), Sh(2), Sc(2), B(2, 0.001)$ & $8.0\%$ & $70.8\%$ & $12.2\%$ & $\mathbf{83.8\%} $\\
\bottomrule
\end{tabular}
\label{table:rasr_dm}
\end{sc}
\end{small}
\end{center}
\end{table*}
\addtolength{\tabcolsep}{4pt}

\section{Common Corruptions}\label{appendix:commoncorruptions}

We also evaluate robust UAP against the 2D fog transformations in \citep{3dcorrupt}. We set the shift intensity of the fog to be 1 and train our robust UAPs to be robust against random fog perturbations. We observe similar results to the transformations we experiment with above. The graph of the results can be seen in Figure \ref{fig:fog}.

\begin{figure*}[ht]

     \centering\includegraphics[width=0.5\textwidth]{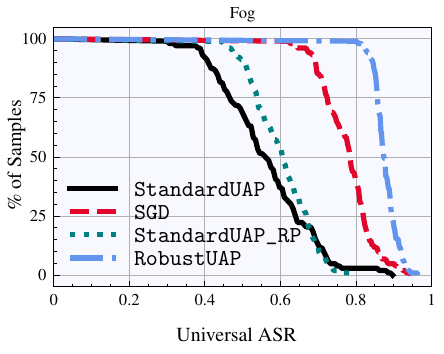}
     
    \caption{For each method, a point $(x,y)$ in the corresponding line represents the percentage of sampled UAPs ($y\%$) with Universal ASR $>x$ for the different semantic transformations on ILSVRC-2012.}
    \label{fig:fog}
\end{figure*}

\section{Algorithm Runtimes}\label{appendix:runtimes}

We compare the average runtimes of the different methods on one of our most challenging $R(10), T(2,2), Sh(2), Sc(2), B(2, 0.001)$ transformation set on ILSVRC-2012 and $n = 738$. The results are in Table \ref{table:runtime}.  We observe that \texttt{RobustUAP} is the slowest algorithm and \texttt{SGD} is the fastest. \texttt{RobustUAP} uses \texttt{EstimateRobustness} in each loop and thus with high $n$ it requires much more time to compute. The extra computation enables \texttt{Robust UAP} to obtain better robustness than all baselines. On the same set of transformations and dataset we observe that one iteration of \texttt{EstimateRobustness} on the entire test set takes on average 19 minutes. When running \texttt{EstimateRobustness} in the \texttt{RobustUAP} loop, each call takes 36 seconds for a batch size of 32. 

\begin{table*}[htb]
\caption{Average Runtime for Robust UAP algorithms}
\begin{center}
\begin{small}
\begin{sc}
\begin{tabular}{lc}
\toprule
Algorithm & time(min) \\
\midrule
\texttt{Standard UAP} & 37\\
\texttt{SGD} & 32\\
\texttt{Standard UAP\_RP} & 43\\
\texttt{Robust UAP} & 118\\
\bottomrule
\end{tabular}
\label{table:runtime}
\end{sc}
\end{small}
\end{center}
\end{table*}

\section{Effect of Compute Time on Robustness}\label{appendix:compute}

Previous sections highlight \texttt{SGD} as the most competitive algorithm to \texttt{RobustUAP} in terms of performance. However, in the previous section we note that \texttt{SGD} takes significantly less time to run. In this section, we investigate how \texttt{RobustUAP} performs with limited compute time as well as how SGD performs with increased runtime. We first add results to the ILSVRC 2012 part of Table \ref{table:rasr} by also computing \texttt{RobustUAP} performance when limited to the same amount of time that \texttt{SGD} takes. Table \ref{table:limited} shows that \texttt{RobustUAP} outperforms \texttt{SGD} even when its compute time is limited with up to 9\% more robustness on our most challenging transformation $R(10), T(2,2), Sh(2), Sc(2), B(2, 0.001)$.

\addtolength{\tabcolsep}{-1.5pt}   
\begin{table*}[ht]
\caption{Robust ASR of \texttt{RobustUAP} restricted to the same amount of compute time as SGD.}
\begin{center}
\begin{small}
\begin{sc}
\begin{tabular}{llcccc}
\toprule
\multirow{2}{*}{Dataset} & \multirow{2}{*}{Transformation Set} & \texttt{SGD} & \texttt{Robust} & \texttt{Restricted} \\
& & & \texttt{UAP} & \texttt{Robust UAP}\\
\midrule
                             &$R(20)$                                    & $69.9\%$ & $\mathbf{93.2}\%$ & $72.9\%$\\
ILSVRC                       &$T(2,2)$                                   & $96.1\%$ & $\mathbf{97.1}\%$ & $96.9\%$\\
2012                         &$Sc(5), R(5), B(5, 0.01)$                  & $85.4\%$ & $\mathbf{96.1}\%$ & $86.3\%$\\
                             &$R(10), T(2,2), Sh(2), Sc(2), B(2, 0.001)$ & $63.1\%$ & $\mathbf{86.4}\%$ & $72.0\%$\\
\bottomrule
\end{tabular}
\label{table:limited}
\end{sc}
\end{small}
\end{center}
\end{table*}
\addtolength{\tabcolsep}{1.5pt}

Next, we vary the number of SGD iterations. We compute the robust ASR on ILSVRC for robustness against $R(10), T(2,2), Sh(2), Sc(2), B(2, 0.001)$. Figure \ref{fig:sgd_performance}, shows the robust ASR achieved by SGD over time, here we observe that SGD's performance flatlines after a small number of iterations and seems to be unable to surpass about $65$. Here \texttt{SGD} is allowed to continue to run past where it would usually stop (at around 250 iterations), in this experiment we allow it to go to 1250 iterations which is about the same amount of time that \texttt{RobustUAP} takes to run. \texttt{RobustUAP} is able to achieve a performance of $72$ even when restricted to the amount of compute time of base \texttt{SGD} (It achieves $86.4$ when unrestricted). These two results in combination show that \texttt{RobustUAP} is able to find more robust UAPs than \texttt{SGD} whose performance stabilizes.

\begin{figure*}[ht]

     \centering\includegraphics[width=0.5\textwidth]{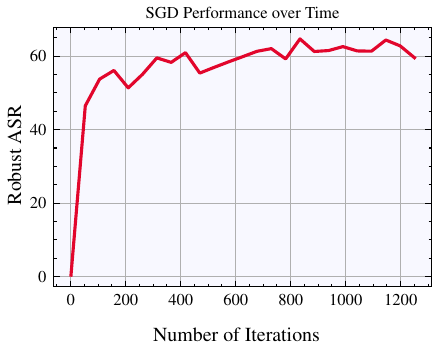}
     
    \caption{The Robust ASR with $\gamma = 0.6$ for \texttt{SGD} over time}
    \label{fig:sgd_performance}
\end{figure*}



\section{Targeted Attack}\label{appendix:targeted}

So far in this paper we have focused on untargeted attacks, i.e. attacks which aim to degrade the general performance of the model. Targeted attacks are also possible with both standard adversarial attack methods and universal adversarial perturbation methods. Here, we can simply turn our algorithm from untargeted to targeted by replacing the loss function. We would like to have target class, A, be classified as target class, B. Instead of maximizing the expected value of the cross entropy loss we can instead formulate the loss based on maximizing B while minimizing A similar to \citep{doubletargetuap}. For ILSVRC 2012, we randomly select a couple of target classes and perform this attack, for each of these cases, we train our robust UAP to be robust to $R(10), T(2,2), Sh(2), Sc(2), B(2, 0.001)$. Table \ref{table:targettarget} shows our results for robust ASR with $\gamma = 0.6$. We are measuring our robust ASR of turning class A into class B and observe similar results with \texttt{RobustUAP} being the most robust followed by \texttt{SGD}. It is also interesting to note that different random combinations lead to more or less success, i.e. it is easier to turn a dog into another dog than perfume into a padlock.

\addtolength{\tabcolsep}{-1.5pt}   
\begin{table*}[ht]
\caption{Robust ASR of \texttt{RobustUAP} for target to target attack compared to the three baselines with $\gamma = 0.6$.}
\begin{center}
\begin{small}
\begin{sc}
\begin{tabular}{llcccc}
\toprule
\multirow{2}{*}{Dataset} & \multirow{2}{*}{Target Class} & \texttt{Standard} & \texttt{SGD} & \texttt{Standard} & \texttt{Robust} \\
& & \texttt{UAP} &  & \texttt{UAP\_RP} & \texttt{UAP} \\
\midrule
\multirow{2}{*}{ILSVRC-2012} &toy poodle $\to$ maltese dog                   & $42.4\%$ & $99.1\%$ & $85.6\%$ & $\mathbf{99.8\%} $\\
                             &perfume $\to$ padlock                          & $0.0\%$ & $63.8\%$ & $5.1\%$ & $\mathbf{76.4\%} $\\
\bottomrule
\end{tabular}
\label{table:targettarget}
\end{sc}
\end{small}
\end{center}
\end{table*}
\addtolength{\tabcolsep}{1.5pt}

\section{\texttt{RobustUAP} vs \texttt{SGD} Performance.} \label{appendix:sgdvsrobust}

When comparing our baselines to \texttt{RobustUAP} we observe that the \texttt{SGD} algorithm has the most comparable performance to \texttt{RobustUAP}. In Appendix \ref{appendix:runtimes}, we report the runtimes for the different algorithms and observe that \texttt{SGD} runs four times as fast as \texttt{RobustUAP}. Although this seems to suggest that \texttt{SGD} is more efficient, we further investigate restricting the runtime of \texttt{RobustUAP} in Appendix \ref{appendix:compute}. We find that \texttt{RobustUAP} beats out \texttt{SGD} at the same runtime. Furthermore, we observe that \texttt{SGD}'s performance flatlines and does not reach the performance of \texttt{RobustUAP} even when allowed to run for longer. Finally, we note that since UAPs only need to be computed a single time and can be done in advance, the runtime considerations are not a big factor for most practical use cases.

\section{Data Efficiency}\label{appendix:dataefficiency}

In this section, we will evaluate the data efficiency of \texttt{RobustUAP}. We use \texttt{RobustUAP} to generate UAPs robust to $R(10), T(2,2), Sh(2), Sc(2), B(2, 0.001)$ on ILSVRC-2012 with differing amounts of training data. The results can be seen in Figure \ref{fig:dataefficiency}. These results show that the algorithm is able to achieve good performance at 500 data points but continues to improve up to 4000 data points. After that it seems to stagnate.

\begin{figure*}[ht]

     \centering\includegraphics[width=0.5\textwidth]{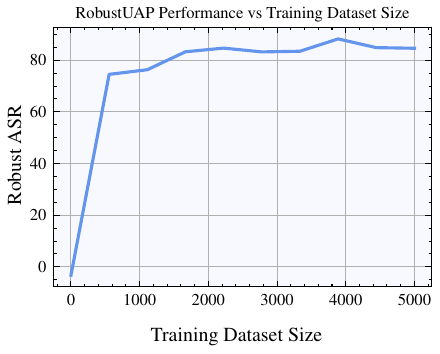}
     
    \caption{Robust ASR with $\gamma = 0.6$ for \texttt{RobustUAP} with differing amounts of training data}
    \label{fig:dataefficiency}
\end{figure*}

\section{Transformer-based Models}\label{appendix:transformers}

Recently, transformers have become popular as a new architecture for deep learning models for computer vision tasks. In this section, we evaluate the effectiveness of robust UAPs against one such model, ViT~\citep{vit}. \citet{vitadversarial} has shown that standard UAPs are still effective against transformer based architectures. In Table \ref{table:transformer} we can see that we get similar results compared to our results on Inception and MobileNet. This shows that our methods work against transformer based models as well.

\addtolength{\tabcolsep}{-4pt}   
\begin{table*}[ht]
\caption{Robust ASR of \texttt{RobustUAP}compared to the three baselines for ViT.}
\begin{center}
\begin{small}
\begin{sc}
\begin{tabular}{lllcccc}
\toprule
\multirow{2}{*}{Dataset} & \multirow{2}{*}{Model} & \multirow{2}{*}{Transformation Set} & \texttt{Standard} & \texttt{SGD} & \texttt{Standard} & \texttt{Robust} \\
& & & \texttt{UAP} &  & \texttt{UAP\_RP} & \texttt{UAP} \\
\midrule
\multirow{1}{*}{ILSVRC-2012}&ViT &$R(10), T(2,2), Sh(2), Sc(2), B(2, 0.001)$   & $2.0\%$ & $72.1\%$ & $12.9\%$ & $\mathbf{88.5\%} $\\
\bottomrule
\end{tabular}
\label{table:transformer}
\end{sc}
\end{small}
\end{center}
\end{table*}
\addtolength{\tabcolsep}{4pt}

\section{Robust UAPs against Robustly Trained  Networks}\label{appendix:robustnetwork}

In this section, we are interested in seeing whether training networks to be robust against the same transformations that the UAP is trying to be robust against is helpful. For this, we trained two new Inception-v3 networks. Because of time limitations, we started with our base Inception-v3 network and fine-tuned it using data augmentations. For the first network InceptionR20, we augmented the data by adding random rotations within 20 degrees. For the second network InceptionHF, we augmented the data by adding horizontal flips. We then crafted UAPs robust against rotations and flips on InceptionR20 and InceptionHF respectively. The results can be seen in Table \ref{table:robustnetworks}. We can compare the $R(20)$ results to those from our normal inception network. We postulate that since the network has received some additional robustness training it is harder to attack, and thus we should see slightly lower robustness scores. However, it seems that training the network to be robust to $R(20)$ does not significantly effect the ability to create robust UAPs. The horizontal flips seems like it might be too easy of a transformation as even standard UAP performs quite well for robust ASR.

\addtolength{\tabcolsep}{-4pt}   
\begin{table*}[ht]
\caption{Robust ASR of \texttt{RobustUAP} compared to the three baselines for robust networks.}
\begin{center}
\begin{small}
\begin{sc}
\begin{tabular}{lllcccc}
\toprule
\multirow{2}{*}{Dataset} & \multirow{2}{*}{Model} & \multirow{2}{*}{Transformation Set} & \texttt{Standard} & \texttt{SGD} & \texttt{Standard} & \texttt{Robust} \\
& & & \texttt{UAP} &  & \texttt{UAP\_RP} & \texttt{UAP} \\
\midrule
\multirow{2}{*}{ILSVRC-2012}&InceptionR20 &$R(20)$   & $6.3\%$ & $72.4\%$ & $10.2\%$ & $\mathbf{81.3\%} $\\
                            &InceptionHF  &$HF$      & $81.3\%$ & $99.5\%$ & $89.7\%$ & $\mathbf{99.6\%} $\\
\bottomrule
\end{tabular}
\label{table:robustnetworks}
\end{sc}
\end{small}
\end{center}
\end{table*}
\addtolength{\tabcolsep}{4pt}

\section{Ablation on optimization strategy}\label{appendix:ablationopt}

In this section, we study the effect of using different optimizers in addition to SGD. We use a variety of standard PyTorch optimizers, Adam, Adamax, Adagrad, and RMSProp. We formulate the optimization problem in the same way but instead use these algorithms in order to optimize our perturbation. We compute these results on ILSVRC-2012 with Inception-v3 and use $R(10), T(2,2), Sh(2), Sc(2), B(2, 0.001)$ as the transformation set and with $\gamma = 0.6$. The results can be seen in Table \ref{table:opt}. We see that the optimization strategy has some affect on the results and that SGD performs the best. We also found that SGD performed marginally faster than the rest of the approaches.

\addtolength{\tabcolsep}{-4pt}   
\begin{table*}[ht]
\caption{Comparison of different optimization strategies.}
\begin{center}
\begin{small}
\begin{sc}
\begin{tabular}{l|c}
\toprule
Optimizer & $ASR_R$\\
\midrule
SGD & $\mathbf{63.1}\%$\\
Adam & $59.7\%$\\
Adamax & $60.1\%$\\
Adagrad & $62.3\%$\\
RMSProp & $58.3\%$\\
\bottomrule
\end{tabular}
\label{table:opt}
\end{sc}
\end{small}
\end{center}
\end{table*}
\addtolength{\tabcolsep}{4pt}

\section{Visualization of Robust UAPs generated with Different Algorithms} \label{appendix:vis}

We further visualize UAPs generated with our three robust algorithms on the same transformation set against a standard UAP generated on ILSVRC 2012 in Figure \ref{fig:comparison}. We observe that \texttt{StandardUAP} UAPs resemble \texttt{StandardUAP\_RP} UAPs, but \texttt{StandardUAP\_RP} algorithm concentrates its budget towards the center as the center is least likely to be perturbed under our transformation set. Both the \texttt{RobustUAP} and the \texttt{SGD} algorithm generate larger patterns across the image. 

\begin{figure}[ht]
\begin{center}
\centerline{\includegraphics[width=0.9\linewidth]{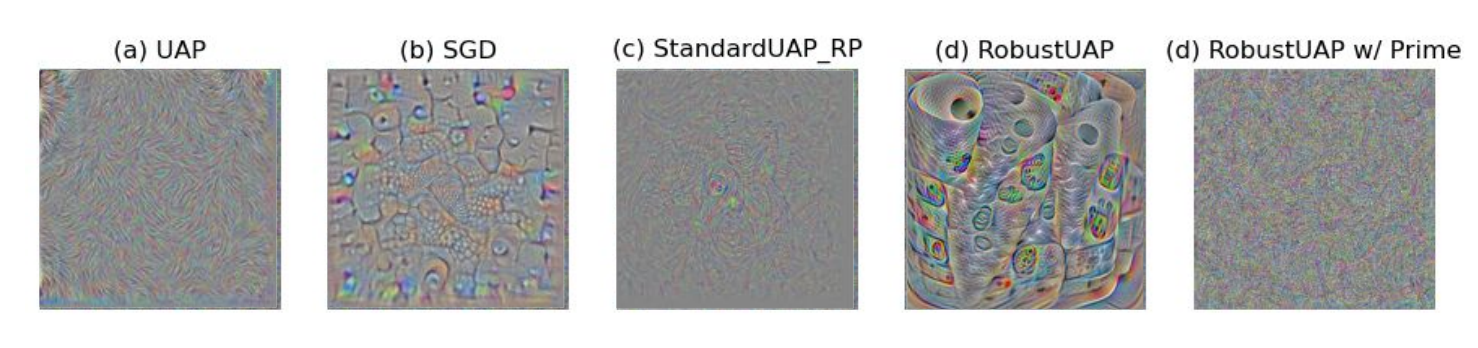}}
\caption{Comparison of UAPs on ILSVRC 2012 generated with (a) \texttt{StandardUAP}, (b) \texttt{RobustUAP}, (c) \texttt{Standard} \texttt{UAP\_RP}, (d) \texttt{RobustUAP}, and (e) \texttt{RobustUAP} generated on Prime.}
\label{fig:comparison}
\end{center}
\end{figure}
\section{Error Bars}\label{appendix:errorbars}

In this section, we will provide error bars/standard deviations for results reported in the paper. 

\subsection{Table \ref{table:rasr}}

First, we report the standard deviations for Table \ref{table:rasr} which we obtain by learning each UAP 10 times then evaluating them on each their respective dataset/transformation set combinations.

\addtolength{\tabcolsep}{-1.5pt}   
\begin{table*}[hbp]
\caption{Standard Deviation of Robust ASR values reported in Table \ref{table:rasr}}
\begin{center}
\begin{footnotesize}
\begin{sc}
\begin{tabular}{llcccc}
\toprule
\multirow{2}{*}{Dataset} & \multirow{2}{*}{Transformation Set} & \texttt{Standard} & \texttt{SGD} & \texttt{Standard} & \texttt{Robust} \\
& & \texttt{UAP} &  & \texttt{UAP\_RP} & \texttt{UAP} \\
\midrule
                             &$R(20)$                                    & $1.1\%$ & $7.2\%$ & $10.2\%$ & $1.5\%$\\
ILSVRC                       &$T(2,2)$                                   & $11.4\%$ & $1.8\%$ & $6.3\%$ & $2.3\%$\\
2012                         &$Sc(5), R(5), B(5, 0.01)$                  & $10.1\%$ & $5.2\%$ & $7.4\%$ & $3.1\%$\\
                             &$R(10), T(2,2), Sh(2), Sc(2), B(2, 0.001)$ & $0.0\%$ & $8.6\%$ & $4.9\%$ & $2.8\%$\\
\midrule
\multirow{3}{*}{CIFAR-10}    &$R(30), B(2, 0.001)$                       & $0.2\%$ & $7.5\%$ & $9.0\%$ & $5.2\%$\\
                             &$R(2), Sh(2)$                              & $11.3\%$ & $5.5\%$ & $8.2\%$ & $0.9\%$\\
                             &$R(10), T(2,2), Sh(2), Sc(2), B(2, 0.001)$ & $1.8\%$ & $6.8\%$ & $5.1\%$ & $3.7\%$\\
\bottomrule
\end{tabular}
\label{table:rasr_std}
\end{sc}
\end{footnotesize}
\end{center}
\end{table*}
\addtolength{\tabcolsep}{1.5pt}

\subsection{Table \ref{table:urasr}}

First, we report the standard deviations for Table \ref{table:urasr} which we obtain by learning each UAP 10 times then evaluating them on each their respective dataset/transformation set combinations.

\addtolength{\tabcolsep}{-4.5pt}   
\begin{table*}[t]
\caption{Robust ASR (\%) of \texttt{RobustUAP} trained on PRIME, Affine ($R(10)$, $T(2,2)$, $Sh(2)$, $Sc(2)$, $B(2, 0.001)$), and Fog when applied to Prime, Affine, and common corruption transforms}
\begin{center}
\begin{scriptsize}
\begin{sc}
\begin{tabular}{c|cc|ccc|cccc|cccc|cccc}
\toprule
 & \multicolumn{17}{c}{Evaluation Corruption Set}\\
\cmidrule{2-18}
Train & \multicolumn{2}{c}{} & \multicolumn{3}{c}{Noise} & \multicolumn{4}{c}{Blur} & \multicolumn{4}{c}{Weather} & \multicolumn{4}{c}{Digital} \\
Set & Prime & Aff. & Gaus. & Shot & Imp. & Defo. & Glass & Moti. & Zoom & Snow & Fog & Frost & Bright & Contr. & Elast. & Pixel & JPEG \\
\midrule
PRIME   & 3.7  & 4.6    & 6.1  & 3.2  & 4.6    & 1.3  & 3.8  & 4.9  & 6.7    & 2.4  & 7.5  & 4.3  & 1.4    & 5.2  & 3.1  & 4.6  & 2.7  \\
Affine  & 9.7  & 5.9    & 3.0  & 2.7  & 5.2    & 8.7  & 7.6  & 5.2  & 0.7    & 11.2 & 6.9  & 6.6  & 5.2    & 0.6  & 3.8  & 9.6  & 6.8  \\
Fog     & 3.1  & 5.2    & 4.7  & 6.6  & 4.6    & 1.9  & 8.4  & 3.5  & 5.1    & 17.8 & 1.6  & 12.1 & 4.0    & 3.6  & 7.3  & 1.1  & 12.6 \\
\bottomrule
\end{tabular}
\label{table:urasr_std}
\end{sc}
\end{scriptsize}
\end{center}
\end{table*}
\addtolength{\tabcolsep}{4.5pt}

\subsection{Remaining Values}

We find that the standard deviations are pretty similar across both tables reported and in some testing of the remaining results. For time reasons we have left the remaining standard deviations out as we don't find them informative. We are happy to provide these numbers for any results in the main body or appendix of the paper.

\end{document}